\documentclass{article}

\usepackage{authblk}
\usepackage{natbib}
\usepackage[margin=1in]{geometry}
\usepackage[utf8]{inputenc} 
\usepackage[T1]{fontenc}    
\usepackage{hyperref} 
\usepackage{url}            
\usepackage{booktabs}       
\usepackage{amsfonts}       
\usepackage{nicefrac}       
\usepackage{microtype}      
\usepackage{xcolor}         

\usepackage{amsmath,amsfonts,amssymb}
\usepackage{amsthm}
\usepackage{bm}
\usepackage{enumitem,array,booktabs}
\usepackage{algorithm}
\usepackage{algorithmic}
\newcounter{protocol}
\makeatletter
\newenvironment{protocol}[1][htb]{%
  \let\c@algorithm\c@protocol
  \renewcommand{\ALG@name}{Protocol}
  \begin{algorithm}[#1]%
  }{\end{algorithm}
}
\makeatother

\usepackage{natbib}
\usepackage{bbm}
\usepackage{tabulary,multirow}
\usepackage{makecell}
\usepackage{dsfont}
\usepackage{tcolorbox}
\usepackage{thmtools}
\usepackage{xpatch}
\xpatchcmd{\proof}{\itshape}{\scshape}{}{}
\usepackage{tablefootnote}

\newtheorem{definition}{Definition}
\newtheorem{lemma}{Lemma}
\newtheorem{theorem}{Theorem}

\newtheorem{example}{Example}
\newtheorem{corollary}{Corollary}

\let\top\intercal

\usepackage[textsize=scriptsize]{todonotes} 
\usepackage{xspace}
\definecolor{misocolor}{rgb}{0.16,0.27,0.86}

\definecolor{graphicbackground}{rgb}{0.96,0.96,0.8}
\definecolor{rouge1}{RGB}{226,0,38}  
\definecolor{orange1}{RGB}{243,154,38}  
\definecolor{jaune}{RGB}{254,205,27}  
\definecolor{blanc}{RGB}{255,255,255} 
\definecolor{rouge2}{RGB}{230,68,57}  
\definecolor{orange2}{RGB}{236,117,40}  
\definecolor{taupe}{RGB}{134,113,127} 
\definecolor{gris}{RGB}{91,94,111} 
\definecolor{bleu1}{RGB}{38,109,131} 
\definecolor{bleu2}{RGB}{28,50,114} 
\definecolor{vert1}{RGB}{133,146,66} 
\definecolor{vert3}{RGB}{20,200,66} 
\definecolor{vert2}{RGB}{157,193,7} 
\definecolor{darkyellow}{RGB}{233,165,0}  
\definecolor{lightgray}{rgb}{0.9,0.9,0.9}
\definecolor{darkgray}{rgb}{0.6,0.6,0.6}
\definecolor{babyblue}{rgb}{0.54, 0.81, 0.94}
\definecolor{citrine}{rgb}{0.89, 0.82, 0.04}
\definecolor{misogreen}{rgb}{0.25,0.6,0.0}
\definecolor{darkmeganta}{rgb}{0.6,0,0.6}

\DeclareMathOperator*{\argmax}{arg\,max}
\DeclareMathOperator*{\argmin}{arg\,min}

\DeclareMathOperator{\loglog}{loglog}

\newcommand{\inner}[2]{\left\langle #1, #2 \right\rangle}

\newcommand{\ceil}[1]{\left\lceil#1\right\rceil}
\newcommand{\floor}[1]{\left\lfloor#1\right\rfloor}

\newcommand{\Ber}{\mathrm{Ber}}
\newcommand{\Rad}{\mathrm{Rad}}
\newcommand{\Unif}{\mathrm{Unif}}

\newtheorem{claim}{Claim}

\newtheorem{problem}{Problem}

\newcommand{\R}{\mathbb{R}}

\newcommand{\NN}{{\mathbb N}}
\newcommand{\1}[1]{\mathds{1}(#1)}
\newcommand{\ind}[1]{\mathds{1}_{#1}}
\newcommand{\true}[1]{\mathds{1}(#1)}
\newcommand{\bOne}{{\bm 1}}
\newcommand{\bZero}{{\bm 0}}

\newcommand{\EE}[1]{\mathbb{E}\left[#1\right]}

\newcommand{\EEs}[2]{\mathbb{E}_{#1}\left[#2\right]}

\newcommand{\norm}[1]{\left\|#1\right\|}

\newcommand{\abs}[1]{\left|#1\right|}

\newcommand{\cA}{\mathcal{A}}
\newcommand{\cB}{\mathcal{B}}

\newcommand{\cD}{\mathcal{D}}

\newcommand{\cH}{\mathcal{H}}

\newcommand{\cL}{\mathcal{L}}

\newcommand{\cM}{\mathcal{M}}

\newcommand{\cO}{\mathcal{O}}

\newcommand{\cQ}{\mathcal{Q}}

\newcommand{\cS}{\mathcal{S}}

\newcommand{\cX}{\mathcal{X}}

\newcommand{\cY}{\mathcal{Y}}

\newcommand{\bP}{{\bf P}}

\newcommand{\be}{{\bf e}}

\renewcommand{\epsilon}{\varepsilon}
\renewcommand{\hat}{\widehat}
\renewcommand{\tilde}{\widetilde}
\renewcommand{\bar}{\overline}

\newcommand{\nothere}[1]{}

\newcommand{\fpr}{\mathrm{fpr}}

\newcommand{\fnr}{\mathrm{fnr}}

\newcommand{\VS}{\mathrm{VS}}

\newcommand{\MB}{\mathrm{MB}}
\newcommand{\SC}{\mathrm{SC}}
\newcommand{\MW}{\mathrm{MWMR}}

\newcommand{\TV}{\mathrm{TV}}
\newcommand{\KL}[2]{\mathrm{D}_\mathrm{KL}(#1\|#2)}
\newcommand{\kl}[2]{\mathrm{kl}(#1, #2)}
\newcommand{\out}{\mathrm{out}}

\newcommand{\loss}{\ell^{\textrm{str}}}
\newcommand{\err}{\cL^{\textrm{str}}}
\newcommand{\blue}[1]{#1}
\newcommand{\red}[1]{#1}

\date{}
\title{Strategic Classification under Unknown Personalized Manipulation}

\author[]{Han Shao}
\author[]{Avrim Blum}
\author[]{Omar Montasser}
\affil[]{Toyota Technological Institute of Chicago\\{ \texttt{\{han,avrim,omar\}@ttic.edu}}}

\begin{document}
\allowdisplaybreaks

\maketitle

\begin{abstract}
We study the fundamental mistake bound and sample complexity in the strategic classification, where agents can strategically manipulate their feature vector up to an extent in order to be predicted as positive. For example, given a classifier determining college admission, student candidates may try to take easier classes to improve their GPA, retake SAT and change schools in an effort to fool the classifier. \textit{Ball manipulations} are a widely studied class of manipulations in the literature, where agents can modify their feature vector within a bounded radius ball. Unlike most prior work, our work considers manipulations to be \textit{personalized}, meaning that agents can have different levels of manipulation abilities (e.g., varying radii for ball manipulations), and \textit{unknown} to the learner.

We formalize the learning problem in an interaction model where the learner first deploys a classifier and the agent manipulates the feature vector within their manipulation set to game the deployed classifier. We investigate various scenarios in terms of the information available to the learner during the interaction, such as observing the original feature vector before or after deployment, observing the manipulated feature vector, or not seeing either the original or the manipulated feature vector. We begin by providing online mistake bounds and PAC sample complexity in these scenarios for ball manipulations. We also explore non-ball manipulations and show that, even in the simplest scenario where both the original and the manipulated feature vectors are revealed, the mistake bounds and sample complexity are lower bounded by $\Omega(|\mathcal{H}|)$ when the target function belongs to a known class $\mathcal{H}$.
\end{abstract}

\section{Introduction}
Strategic classification addresses the problem of learning a classifier robust to manipulation and gaming by self-interested agents~\citep{hardt2016strategic}.
For example, given a classifier determining loan approval based on credit scores, applicants could open or close credit cards and bank accounts to increase their credit scores. In the case of a college admission classifier, students may try to take easier classes to improve their GPA, retake the SAT or change schools in an effort to be admitted.
In both cases, such manipulations do not change their true qualifications.
Recently, a collection of papers has studied strategic classification in both the online setting where examples are chosen by an adversary in a sequential manner~\citep{dong2018strategic,chen2020learning,ahmadi2021strategic, ahmadi2023fundamental}, and the distributional setting where the examples are drawn from an underlying data distribution~\citep{hardt2016strategic,zhang2021incentive,sundaram2021pac,lechner2022learning}.
Most existing works assume that manipulation ability is uniform across all agents or is known to the learner. However, in reality, this may not always be the case.
For instance, low-income students may have a lower ability to manipulate the system compared to their wealthier peers due to factors such as the high costs of retaking the SAT or enrolling in additional classes, as well as facing more barriers to accessing information about college~\citep{milli2019social} and it is impossible for the learner to know the highest achievable GPA or the maximum number of times a student may retake the SAT due to external factors such as socio-economic background and personal circumstances.

We characterize the manipulation of an agent by a set of alternative feature vectors that she can modify her original feature vector to, which we refer to as the \textit{manipulation set}.
\textit{Ball manipulations} are a widely studied class of manipulations in the literature, where agents can modify their feature vector within a bounded radius ball.
For example, \cite{dong2018strategic,chen2020learning,sundaram2021pac} studied ball manipulations with distance function being some norm and \cite{zhang2021incentive,lechner2022learning,ahmadi2023fundamental} studied a manipulation graph setting, which can be viewed as ball manipulation w.r.t. the graph distance on a predefined known graph.


In the online learning setting, the strategic agents come sequentially and try to game the current classifier.
Following previous work, we model the learning process as a repeated Stackelberg game over $T$ time steps.
In round $t$, the learner proposes a classifier $f_t$ and then the agent, with a manipulation set (unknown 
 to the learner), manipulates her feature in an effort to receive positive prediction from $f_t$.
There are several settings based on what and when the information is revealed about the original feature vector and the manipulated feature vector in the game.
The simplest setting for the learner is observing the original feature vector before choosing $f_t$ and the manipulated vector after.
In a slightly harder setting, the learner observes both the original and manipulated vectors after selecting $f_t$.
An even harder setting involves observing only the manipulated feature vector after selecting $f_t$.
The hardest and least informative scenario occurs when neither the original nor the manipulated feature vectors are observed.

In the distributional setting, the agents are sampled from an underlying data distribution.
Previous work assumes that the learner has full knowledge of the original feature vector and the manipulation set, and then views learning as a one-shot game and solves it by computing the Stackelberg equilibria of it.
However, when manipulations are personalized and unknown, we cannot compute an equilibrium and study learning as a one-shot game. In this work, we extend the iterative online interaction model from the online setting to the distributional setting, where the sequence of agents is sampled i.i.d. from the data distribution.
After repeated learning for $T$ (which is equal to the sample size) rounds, the learner has to output a strategy-robust predictor for future use.

In both online and distributional settings, examples are viewed through the lens of the current predictor and the learner does not have the ability to inquire about the strategies the previous examples would have adopted under a different predictor.



\paragraph{Related work}
Our work is primarily related to strategic classification in online and distributional settings.
Strategic classification was first studied in a distributional model by~\cite{hardt2016strategic} and subsequently by~\cite{dong2018strategic} in an online model.
\cite{hardt2016strategic} assumed that agents manipulate by best response with respect to a uniform cost function known to the learner.
Building on the framework of \citep{hardt2016strategic}, \cite{lechner2022learning,sundaram2021pac,zhang2021incentive,hu2019disparate,milli2019social} studied the distributional learning problem, and all of them assumed that the manipulations are predefined and known to the learner, either by a cost function or a predefined manipulation graph.
For online learning, \cite{dong2018strategic} considered a similar manipulation setting as in this work, where manipulations are personalized and unknown. However, they studied linear classification with ball manipulations in the online setting and focused on finding appropriate conditions of the cost function to achieve sub-linear Stackelberg regret.
\cite{chen2020learning} also studied Stackelberg regret in linear classification with uniform ball manipulations. \cite{ahmadi2021strategic} studied the mistake bound under uniform (possbily unknown) ball manipulations, and \cite{ahmadi2023fundamental} studied regret under a pre-defined and known manipulation.
The most relevant work is a recent concurrent study by~\cite{lechner2023strategic}, which also explores strategic classification involving unknown personalized manipulations but with a different loss function. 
In their work, a predictor incurs a loss of $0$ if and only if the agent refrains from manipulation and the predictor correctly predicts at the unmanipulated feature vector. 
In our work, the predictor's loss is $0$ if it correctly predicts at the manipulated feature, even when the agent manipulates.
As a result, their loss function serves as an upper bound of our loss function.

There has been a lot of research on various other issues and models in strategic classification. 
Beyond sample complexity, \cite{hu2019disparate,milli2019social} focused on other social objectives, such as social burden and fairness. 
Recent works also explored different models of agent behavior, including
proactive agents~\cite{zrnic2021leads}, non-myopic agents~\citep{haghtalab2022learning} and noisy agents~\citep{jagadeesan2021alternative}.
\cite{ahmadi2023fundamental} considers two agent models of randomized learners: a randomized algorithm model where the agents respond to the realization, and a fractional classifier model where agents respond to the expectation, and our model corresponds to the randomized algorithm model.
Additionally, there is also a line of research on agents interested in improving their qualifications instead of gaming~\citep{kleinberg2020classifiers, haghtalab2020maximizing,ahmadi2022classification}.
Strategic interactions in the regression setting have also been studied (e.g.,~\cite{bechavod2021gaming}).

Beyond strategic classification, there is a more general research area of learning using data from strategic sources, such as a single data generation player who manipulates the data distribution~\citep{bruckner2011stackelberg,dalvi2004adversarial}. Adversarial perturbations can be viewed as another type of strategic source~\citep{montasser2019vc}.




\section{Model}\label{sec:model}
\paragraph{Strategic classification} 
Throughout this work, we consider the binary classification task.
Let $\cX$ denote the feature vector space, $\cY = \{+1,-1\}$ denote the label space, and $\cH\subseteq \cY^\cX$ denote the hypothesis class.
In the strategic setting, instead of an example being a pair $(x,y)$, an example, or \textit{agent}, is a triple $(x,u,y)$ where $x \in \cX$ is the original feature vector, $y \in \cY$ is the label, and $u \subseteq \cX$ is the manipulation set, which is a set of feature vectors that the agent can modify their original feature vector $x$ to. 
In particular, given a hypothesis $h\in \cY^\cX$, the agent will try to manipulate her feature vector $x$ to another feature vector $x'$ within $u$ in order to receive a positive prediction from $h$.
The manipulation set $u$ is \textit{unknown} to the learner.
In this work, we will be considering several settings based on what the information is revealed to the learner, including both the original/manipulated feature vectors, the manipulated feature vector only, or neither, and when the information is revealed.



More formally, for agent $(x,u,y)$, given a predictor $h$, if $h(x) = -1$ and her manipulation set overlaps the positive region by $h$, i.e., $u \cap \cX_{h,+}\neq \emptyset$ with $\cX_{h,+}:= \{x\in \cX|h(x)=+1\}$, the agent will manipulate $x$ to $\Delta(x,h,u)\in u\cap \cX_{h,+}$\footnote{For ball manipulations, agents break ties by selecting the closest vector. When there are multiple closest vectors, agents break ties arbitrarily. For non-ball manipulations, agents break ties in any fixed way.
} to receive positive prediction by $h$.
Otherwise, the agent will do nothing and maintain her feature vector at $x$, i.e., $\Delta(x,h,u)=x$.
We call $\Delta(x,h,u)$ the manipulated feature vector of agent $(x,u,y)$ under predictor $h$.



A general and fundamental type of manipulations is \textit{ball manipulations}, where agents can manipulate their feature within a ball of \textit{personalized} radius.
More specifically, given a metric $d$ over $\cX$, the manipulation set is a ball $\cB(x;r) = \{x'|d(x,x')\leq r\}$ centered at $x$ with radius $r$ for some $r\in \R_{\geq 0}$.
Note that we allow different agents to have different manipulation power and the radius can vary over agents.
Let $Q$ denote the set of allowed pairs $(x,u)$, which we refer to as the feature-manipulation set space.
For ball manipulations, we have
$\cQ = \{(x, \cB(x;r))|x\in \cX, r\in \R_{\geq 0}\}$ for some known metric $d$ over $\cX$. In the context of ball manipulations, we use $(x,r,y)$ to represent $(x,\cB(x;r),y)$ and $\Delta(x,h,r)$ to represent $\Delta(x,h,\cB(x;r))$ for notation simplicity.

For any hypothesis $h$, let the strategic loss $\loss(h,(x,u,y))$ of $h$ be defined as the loss at the manipulated feature, i.e., $\loss(h,(x,u,y)) := \1{h(\Delta(x,h,u))\neq y}$. According to our definition of $\Delta(\cdot)$, we can write down the strategic loss explicitly as 
\begin{align}
    \loss(h,(x,u,y)) = \begin{cases}
        1 & \text{ if } y = -1, h(x) = +1 \\
        1 & \text{ if } y = -1, h(x)=-1 \text{ and } u \cap \cX_{h,+}\neq \emptyset\,,  \\
        1 & \text{ if } y = +1, h(x) = -1\text{ and } u \cap \cX_{h,+}= \emptyset\,,  \\
        0 & \text{ otherwise.} 
    \end{cases}
\end{align}
For any randomized predictor $p$ (a distribution over hypotheses), the strategic behavior depends on the realization of the predictor and the strategic loss of $p$ is $\loss(p,(x,u,y)) := \EEs{h\sim p}{\loss(h,(x,u,y))}$.

\paragraph{Online learning}
We consider the task of sequential classification where the learner aims to classify a sequence of agents $(x_1,u_1,y_1),(x_2,u_2,y_2),\ldots, (x_T,u_T,y_T)\in \cQ\times \cY$ that arrives in an online manner.
At each round, the learner feeds a predictor to the environment and then observes his prediction $\hat y_t$, the true label $y_t$ and possibly along with some additional information about the original/manipulated feature vectors.
We say the learner makes a mistake at round $t$ if $\hat y_t\neq y_t$ and the learner's goal is to minimize the number of mistakes on the sequence.
The interaction protocol (which repeats for $t=1,\ldots,T$) is described in the following.
\vspace{-1mm}
\begin{protocol}[H]
    \caption{Learner-Agent Interaction at round $t$}
    \label{prot:interaction}
        \begin{algorithmic}[1]
            \STATE The environment picks an agent $(x_t,u_t,y_t)$ and reveals some context $C(x_t)$. In the online setting, the agent is chosen adversarially, while in the distributional setting, the agent is sampled i.i.d.
            \STATE The learner $\cA$ observes $C(x_t)$ and picks a hypothesis $f_t\in \cY^\cX$.
            \STATE The learner $\cA$ observes the true label $y_t$, the prediction $\hat y_t = f_t(\Delta_t)$, and some feedback $F(x_t,\Delta_t)$, where $\Delta_t = \Delta(x_t,f_t,u_t)$ is the manipulated feature vector.
        \end{algorithmic}
    \end{protocol}
\vspace{-4mm}
The context function \blue{$C(\cdot)$} and feedback function \red{$F(\cdot)$} reveals information about the original feature vector $x_t$ and the manipulated feature vector $\Delta_t$. \blue{$C(\cdot)$} reveals the information before the learner picks $f_t$ while \red{$F(\cdot)$} does after.
We study several different settings based on what and when information is revealed.
\begin{itemize}[leftmargin = *]
    \item The simplest setting for the learner is observing the original feature vector $x_t$ before choosing $f_t$ and the manipulated vector $\Delta_t$ after. Consider a teacher giving students a writing assignment or take-home exam. The teacher might have a good knowledge of the students' abilities (which correspond to the original feature vector $x_t$) based on their performance in class, but the grade has to be based on how well they do the assignment. The students might manipulate by using the help of ChatGPT / Google / WolframAlpha / their parents, etc. The teacher wants to create an assignment that will work well even in the presence of these manipulation tools. In addition, If we think of each example as representing a subpopulation (e.g., an organization is thinking of offering loans to a certain group), then there might be known statistics about that population, even though the individual classification (loan) decisions have to be made based on responses to the classifier.
    This setting corresponds to \blue{$C(x_t) =x_t$} and \red{$F(x_t,\Delta_t) = \Delta_t$}.
    We denote a setting by their values of $C,F$ and thus, we denote this setting by $(x,\Delta)$.
    \item In a slightly harder setting, the learner observes both the original and manipulated vectors after selecting $f_t$ and thus, $f_t$ cannot depend on the original feature vector in this case. For example, if a high-school student takes the SAT test multiple times, most colleges promise to only consider the highest one (or even to "superscore" the test by considering the highest score separately in each section) but they do require the student to submit all of them. Then \blue{$C(x_t) =\perp$} and \red{$F(x_t,\Delta_t) = (x_t,\Delta_t)$}, where $\perp$ is a token for ``no information'', and this setting is denoted by $(\perp, (x,\Delta))$.
    \item An even harder setting involves observing only the manipulated feature vector after selecting $f_t$ (which can only be revealed after $f_t$ since $\Delta_t$ depends on $f_t$). Then \blue{$C(x_t) =\perp$} and \red{$F(x_t,\Delta_t) = \Delta_t$} and this setting is denoted by $(\perp, \Delta)$.
    \item The hardest and least informative scenario occurs when neither the original nor the manipulated feature vectors are observed. Then \blue{$C(x_t) =\perp$} and \red{$F(x_t,\Delta_t) = \perp$} and it is denoted by $(\perp, \perp)$.


\end{itemize}
Throughout this work, we focus on the \textit{realizable} setting, where there exists a perfect classifier in $\cH$ that never makes any mistake at the sequence of strategic agents. 
More specifically, there exists a hypothesis $h^*\in \cH$ such that for any $t\in [T]$, we have $y_t = h^*(\Delta(x_t,h^*,u_t))$\footnote{It is possible that 
there is no hypothesis $\bar h\in \cY^\cX$ s.t. $y_t = \bar h(x_t)$ for all $t\in [T]$.}.
Then we define the mistake bound as follows.
\begin{definition}
For any choice of $(C,F)$, let $\cA$ be an online learning algorithm under Protocol~\ref{prot:interaction} in the setting of $(C,F)$.
Given any realizable sequence $S = ((x_1,u_1, h^*(\Delta(x_1,h^*,u_1))),\ldots, (x_T,u_T, h^*(\Delta(x_T,h^*,u_T)))\in (\cQ\times \cY)^T$, where $T$ is any integer and $h^*\in \cH$, 
let $\cM_{\cA}(S)$ be the number of mistakes $\cA$ makes on the sequence $S$.
The mistake bound of $(\cH,\cQ)$, denoted $\MB_{C,F}$, is the smallest number $B\in \NN$ such that there exists an algorithm $\cA$ such that $\cM_{\cA}(S)\leq B$ over all realizable sequences $S$ of the above form.
\end{definition}
According the rank of difficulty of the four settings with different choices of $(C,F)$, the mistake bounds are ranked in the order of $\MB_{\blue{x},\red{\Delta}} \leq \MB_{\blue{\perp},\red{(x,\Delta)}} \leq \MB_{\blue{\perp},\red{\Delta}} \leq \MB_{\blue{\perp},\red{\perp}}$.


\paragraph{PAC learning} 
In the distributional setting, the agents are sampled from an underlying distribution $\cD$ over $\cQ\times \cY$.
The learner's goal is to find a hypothesis $h$ with low population loss $\err_\cD(h) := \EEs{(x,u,y)\sim \cD}{\loss(h,(x,u,y))}$. 
One may think of running empirical risk minimizer (ERM) over samples drawn from the underlying data distribution, i.e., returning $\argmin_{h\in \cH} \frac{1}{m}\sum_{i=1}^m \loss(h,(x_i,u_i,y_i))$, where $(x_1,u_1,y_1), \ldots,(x_m,u_m,y_m)$ are i.i.d. sampled from $\cD$.
However, ERM is unimplementable because the manipulation sets $u_i$'s are never revealed to the algorithm, and only the partial feedback in response to the implemented classifier is provided.  In particular, 
in this work we consider using the same interaction protocol as in the online setting, i.e., Protocol~\ref{prot:interaction}, with agents $(x_t,u_t,y_t)$ i.i.d. sampled from the data distribution $\cD$.
After $T$ rounds of interaction (i.e., $T$ i.i.d. agents), the learner has to output a predictor $f_\out$ for future use. 

Again, we focus on the \textit{realizable} setting, where the sequence of sampled agents (with manipulation) can be perfectly classified by a target function in $\cH$.
Alternatively, there exists a classifier with zero population loss, i.e., there exists a hypothesis $h^*\in \cH$ such that $\err_\cD(h^*) = 0$. Then we formalize the notion of PAC sample complexity under strategic behavior as follows.
\begin{definition}
For any choice of $(C,F)$, let $\cA$ be a learning algorithm that interacts with agents using Protocol~\ref{prot:interaction} in the setting of $(C,F)$ and outputs a predictor $f_\out$ in the end.
For any $\epsilon,\delta \in (0,1)$, the sample complexity of realizable $(\epsilon,\delta)$-PAC learning of $(\cH,\cQ)$, denoted $\SC_{C,F}(\epsilon,\delta)$, is defined as the smallest $m\in \NN$ for which there exists a learning algorithm $\cA$ in the above form such that for any distribution $\cD$ over $\cQ\times \cY$ where there exists a predictor $h^*\in \cH$ with zero loss, $\err_\cD(h) = 0$, with probability at least $1-\delta$ over $(x_1,u_1, y_1),\ldots,  (x_m,u_m,y_m)\stackrel{\text{i.i.d.}}{\sim} \cD$, $\err_\cD(f_\out)\leq \epsilon$.
\end{definition}
Similar to mistake bounds, the sample complexities are ranked in the same order $\SC_{x,\Delta} \leq \SC_{\perp,(x,\Delta)} \leq \SC_{\perp,\Delta} \leq \SC_{\perp,\perp}$ according to the rank of difficulty of the four settings.

\section{Overview of Results}\label{sec:overview}
In classic (non-strategic) online learning, the Halving algorithm achieves a mistake bound of $\log(\abs{\cH})$ by employing the majority vote and eliminating inconsistent hypotheses at each round.
In classic PAC learning, the sample complexity of $\cO(\frac{\log(\abs{\cH})}{\epsilon})$ is achievable via ERM.
Both mistake bound and sample complexity exhibit logarithmic dependency on $\abs{\cH}$.
This logarithmic dependency on $\abs{\cH}$ (when there is no further structural assumptions) is tight in both settings, i.e., there exist examples of $\cH$ with mistake bound of $\Omega(\log(\abs{\cH}))$ and with sample complexity of $\Omega(\frac{\log(\abs{\cH})}{\epsilon})$.
In the setting where manipulation is known beforehand and only $\Delta_t$ is observed, \cite{ahmadi2023fundamental} proved a lower bound of $\Omega(\abs{\cH})$ for the mistake bound.
Since in the strategic setting we can achieve a linear dependency on $\abs{\cH}$ by trying each hypothesis in $\cH$ one by one and discarding it once it makes a mistake, the question arises:
\begin{center}
    \textbf{\textit{Can we achieve a logarithmic dependency on $\abs{\cH}$ in strategic classification?}}
\end{center}
In this work, we show that the dependency on $\abs{\cH}$ varies across different settings and that in some settings mistake bound and PAC sample complexity can exhibit different dependencies on $\abs{\cH}$.
We start by presenting our results for ball manipulations in the four settings.
\begin{itemize}[leftmargin = *]
    \item Setting of $(x,\Delta)$
    (observing $x_t$ before choosing $f_t$ and observing $\Delta_t$ after)
    : For online learning, we propose an variant of the Halving algorithm, called Strategic Halving (Algorithm~\ref{alg:halving}), which can eliminate half of the remaining hypotheses when making a mistake. The algorithm depends on observing $x_t$ before choosing the predictor $f_t$.
    Then by applying the standard technique of converting mistake bound to PAC bound, we are able to achieve sample complexity of $\cO(\frac{\log(\abs{\cH})\loglog(\abs{\cH})}{\epsilon})$.
    \item Setting of $(\perp,(x,\Delta))$
    (observing both $x_t$ and $\Delta_t$ after selecting $f_t$)
    : 
    We prove that, there exists an example of $(\cH,\cQ)$ s.t. the mistake bound is lower bounded by $\Omega(\abs{\cH})$. 
   This implies that no algorithm can perform significantly better than sequentially trying each hypothesis, which would make at most $\abs{\cH}$ mistakes before finding the correct hypothesis.
   However, unlike the construction of mistake lower bounds in classic online learning, where all mistakes can be forced to occur in the initial rounds, we demonstrate that we require $\Theta(\abs{\cH}^2)$ rounds to ensure that all mistakes occur.
    In the PAC setting, we first show that, any learning algorithm with proper output $f_\out$, i.e., $f_\out\in\cH$, needs a sample size of $\Omega(\frac{\abs{\cH}}{\epsilon})$. 
    We can achieve a sample complexity of $O(\frac{\log^2(\abs{\cH})}{\epsilon})$ by executing Algorithm~\ref{alg:end-iid-ball}, which is a randomized algorithm with improper output.

    \item Setting of $(\perp,\Delta)$
    (observing only $\Delta_t$ after selecting $f_t$)
    : The mistake bound of $\Omega(\abs{\cH})$ also holds in this setting, as it is known to be harder than the previous setting. For the PAC learning, we show that any conservative algorithm, which only depends on the information from the mistake rounds, requires $\Omega(\frac{\abs{\cH}}{\epsilon})$ samples.
    The optimal sample complexity is left as an open problem.

    \item Setting of $(\perp,\perp)$
    (observing neither $x_t$ nor $\Delta_t$)
    : Similarly, the mistake bound of $\Omega(\abs{\cH})$ still holds.
    For the PAC learning, we show that the sample complexity is $\Omega(\frac{\abs{\cH}}{\epsilon})$ by reducing the problem to a stochastic linear bandit problem.
\end{itemize}

Then we move on to non-ball manipulations. However, we show that even in the simplest setting of observing $x_t$ before choosing $f_t$ and observing $\Delta_t$ after, there is an example of $(\cH,\cQ)$ such that the sample complexity is $\tilde\Omega(\frac{\abs{\cH}}{\epsilon})$. This implies that in all four settings of different revealed information, we will have sample complexity of $\tilde\Omega(\frac{\abs{\cH}}{\epsilon})$ and mistake bound of $\tilde\Omega(\abs{\cH})$.
We summarize our results in Table~\ref{tab:res}.

\begin{table}[htb!]
    \centering
    {
    \renewcommand{\arraystretch}{1.5}
    \begin{tabular}{c|c|c|c}
        & setting & mistake bound & sample complexity\\\hline
       \multirow{5}{*}{ball} & $(x,\Delta)$
       &  $ \Theta(\log(\abs{\cH}))$ (Thm~\ref{thm:halving}) & $\tilde \cO(\frac{\log(\abs{\cH})}{\epsilon})$$^a$ (Thm~\ref{thm:x-first-pac}), $\Omega(\frac{\log(\abs{\cH})}{\epsilon})$\\\cline{2-4}
        &\multirow{2}{*}{ $(\perp,(x,\Delta))$
        } & $\cO(\min(\sqrt{ \log(\abs{\cH})T},\abs{\cH}))$ (Thm~\ref{thm:mw})
         & $\cO(\frac{\log^2(\abs{\cH})}{\epsilon})$ (Thm~\ref{thm:x_end_pac}), $\Omega(\frac{\log(\abs{\cH})}{\epsilon})$\\
        & &
        $\Omega(\min(\frac{T}{\abs{\cH}\log(\abs{\cH})}, \abs{\cH}))$(Thm~\ref{thm:x-end-online})&  $\SC^{\text{prop}}= \Omega(\frac{\abs{\cH}}{\epsilon})$ (Thm~\ref{thm:x-end-pac-proper})
        \\\cline{2-4}
        & $(\perp,\Delta)$
        & $\Theta(\abs{\cH})$ (implied by Thm~\ref{thm:x-end-online}) & $\SC^{\text{csv}} = \tilde\Omega(\frac{\abs{\cH}}{\epsilon})$ (Thm~\ref{thm:delta-csv})
        \\\cline{2-4}
       & $(\perp,\perp)$
       & $\Theta(\abs{\cH})$ (implied by Thm~\ref{thm:x-end-online}) & $\tilde \cO(\frac{\abs{\cH}}{\epsilon})$ , $ \tilde\Omega(\frac{\abs{\cH}}{\epsilon})$ (Thm~\ref{thm:x-delta-never})
        \\\hline
       nonball & all & $ \tilde \Omega(\abs{\cH})$(Cor~\ref{cor:non-ball-all}) , $\cO(\abs{\cH})$ & $\tilde \cO(\frac{\abs{\cH}}{\epsilon})$ , $ \tilde\Omega(\frac{\abs{\cH}}{\epsilon})$ (Cor~\ref{cor:non-ball-all})\\\hline
    \end{tabular}}
    
    \raggedright \footnotesize{$^a$ A factor of $\loglog(\abs{\cH})$ is neglected.}
    \caption{The summary of results. $\tilde\cO$ and $\tilde \Omega$ ignore logarithmic factors on $\abs{\cH}$ and $\frac{1}{\epsilon}$.
    The superscripts prop stands for proper learning algorithms and csv stands for conservative learning algorithms. 
    All lower bounds in the non-strategic setting also apply to the strategic setting, implying that $\MB_{C,F}\geq \Omega(\log(\abs{\cH}))$ and $\SC_{C,F}\geq \Omega(\frac{\log(\abs{\cH})}{\epsilon})$ for all settings of $(C,F)$.
    In all four settings, a mistake bound of $\cO(\abs{\cH})$ can be achieved by simply trying each hypothesis in $\cH$ while the sample complexity can be achieved as $\tilde \cO(\frac{\abs{\cH}}{\epsilon})$ by converting the mistake bound of $\cO(\abs{\cH})$ to a PAC bound using standard techniques.}
    \label{tab:res}
    \vspace{-10pt}
\end{table}

\section{Ball manipulations}
In ball manipulations, when $\cB(x;r) \cap \cX_{h,+}$ has multiple elements, the agent will always break ties by selecting the one closest to $x$, i.e., $\Delta(x,h, r) = \argmin_{x'\in \cB(x;r)  \cap \cX_{h,+}} d(x,x')$.
In round $t$, the learner deploys predictor $f_t$, and once he knows $x_t$ and $\hat y_t$, he can calculate $\Delta_t$ himself without needing knowledge of $r_t$ by
\vspace{-5pt}
\begin{align*}
    \Delta_t = \begin{cases}
    \argmin_{x'\in \cX_{f_t,+}} d(x_t,x') & \text{ if } \hat y_t = +1\,,\\
    x_t & \text{ if } \hat y_t = -1\,.
    \end{cases}
\end{align*}
Thus, for ball manipulations, knowing $x_t$ is equivalent to knowing both $x_t$ and $\Delta_t$.
\subsection{Setting $(x,\Delta)$: Observing $x_t$ Before Choosing $f_t$} 
\paragraph{Online learning} We propose a new algorithm with mistake bound of $\log(\abs{\cH})$ in setting $(x,\Delta)$.
To achieve a logarithmic mistake bound, we must construct a predictor $f_t$ such that if it makes a mistake, we can reduce a constant fraction of the remaining hypotheses.
The primary challenge is that we do not have access to the full information, and predictions of other hypotheses are hidden.
To extract the information of predictions of other hypotheses, we take advantage of ball manipulations, which induces an ordering over all hypotheses.
Specifically, for any hypothesis $h$ and feature vector $x$, we define the distance between $x$ and $h$ by the distance between $x$ and the positive region by $h$, $\cX_{h,+}$, i.e.,
\begin{equation}
    d(x,h) := \min\{ d(x,x')|x'\in \cX_{h,+}\}\,.\label{eq:dist}
\end{equation}
At each round $t$, given $x_t$, the learner calculates the distance $d(x_t,h)$ for all $h$ in the version space (meaning hypotheses consistent with history) and selects a hypothesis $f_t$ such that $d(x_t,f_t)$ is the median among all distances $d(x_t,h)$ for $h$ in the version space.
We can show that by selecting $f_t$ in this way, the learner can eliminate half of the version space if $f_t$ makes a mistake.
We refer to this algorithm as Strategic Halving, and provide a detailed description of it in Algorithm~\ref{alg:halving}.
\begin{theorem}\label{thm:halving}
    For any feature-ball manipulation set space $\cQ$ and hypothesis class $\cH$, Strategic Halving achieves mistake bound $\MB_{x,\Delta} \leq \log(\abs{\cH})$.
\end{theorem}
\vspace{-7pt}
\begin{algorithm}[H]\caption{Strategic Halving}\label{alg:halving}
    \begin{algorithmic}[1]
    \STATE Initialize the version space $\VS=\cH$.
    \FOR{$t=1,\ldots,T$}
    \STATE pick an $f_t\in \VS$ such that $d(x_t,f_t)$ is the median of $\{d(x_t,h)|h\in \VS\}$.\label{algline:pickf}
    \STATE \textbf{if} {$\hat y_t \neq y_t$ and $y_t = +$} \textbf{then}
     $\VS\leftarrow \VS\setminus \{h\in \VS|d(x_t,h)\geq d(x_t,f_t)\}$;
    \STATE \textbf{else if} {$\hat y_t \neq y_t$ and $y_t = -$} \textbf{then} $\VS\leftarrow  \VS\setminus \{h\in \VS|d(x_t,h)\leq d(x_t,f_t)\}$.
    \ENDFOR
    \end{algorithmic}
\end{algorithm}
\vspace{-11pt}
To prove Theorem~\ref{thm:halving}, we only need to show that each mistake reduces the version space by half.
Supposing that $f_t$ misclassifies a true positive example $(x_t,r_t, +1)$ by negative, then we know that $d(x_t, f_t)> r_t$ while the target hypothesis $h^*$ must satisfy that $d(x_t, h^*)\leq r_t$.
Hence any $h$ with $d(x_t,h)\geq d(x_t,f_t)$ cannot be $h^*$ and should be eliminated. Since $d(x_t,f_t)$ is the median of $\{d(x_t,h)|h\in \VS\}$, we can elimate half of the version space.
It is similar when $f_t$ misclassifies a true negative.
The detailed proof is deferred to Appendix~\ref{app:halving}.

\paragraph{PAC learning} 
We can convert Strategic Halving to a PAC learner by the standard technique of converting a mistake bound to a PAC bound \citep{10008965845}. Specifically, the learner runs Strategic Halving until it produces a hypothesis $f_t$ that survives for $\frac{1}{\epsilon}\log(\frac{\log(\abs{\cH})}{\delta})$ rounds and outputs this $f_t$.
Then we have Theorem~\ref{thm:x-first-pac}, and the proof is included in Appendix~\ref{app:x_firtst_pac}.
\begin{theorem}\label{thm:x-first-pac}
    For any feature-ball manipulation set space $\cQ$ and hypothesis class $\cH$, we can achieve $\SC_{x,\Delta}(\epsilon,\delta) = \cO(\frac{\log(\abs{\cH})}{\epsilon}\log(\frac{\log(\abs{\cH})}{\delta}))$ by combining Strategic Halving and the standard technique of converting a mistake bound to a PAC bound.
\end{theorem}



\subsection{Setting $(\perp, (x,\Delta))$: Observing $x_t$ After Choosing $f_t$} 
When $x_t$ is not revealed before the learner choosing $f_t$, the algorithm of Strategic Halving does not work anymore. 
We demonstrate that it is impossible to reduce constant fraction of version space when making a mistake, and 
prove that the mistake bound is lower bounded by $\Omega(\abs{\cH})$ by constructing a negative example of $(\cH,\cQ)$.
However, we can still achieve sample complexity with poly-logarithmic dependency on $\abs{\cH}$ in the distributional setting.
\subsubsection{Results in the Online Learning Model}
To offer readers an intuitive understanding of the distinctions between the strategic setting and standard online learning, we commence by presenting an example in which no deterministic learners, including the Halving algorithm, can make fewer than $\abs{\cH} -1$ mistakes.

\begin{example}\label{eg:x-after-det}
Consider a star shape metric space $(\cX,d)$, where $\cX = \{0,1,\ldots,n\}$, $d(i,j) = 2$ and $d(0,i)=1$  for all $i,j\in [n]$ with $i\neq j$.
The hypothesis class is composed of singletons over $[n]$, i.e., $\cH = \{2\ind{\{i\}}-1|i\in [n]\}$.
When the learner is deterministic, the environment can pick an agent $(x_t, r_t, y_t)$ dependent on $f_t$.
If $f_t$ is all-negative, then the environment picks $(x_t,r_t,y_t) = (0,1,+1)$, and then the learner makes a mistake but no hypothesis can be eliminated. 
If $f_t$ predicts $0$ by positive, the environment will pick $(x_t,r_t,y_t) = (0,0,-1)$, and then the learner makes a mistake but no hypothesis can be eliminated. 
If $f_t$ predicts some $i\in [n]$ by positive, the environment will pick $(x_t,r_t,y_t) = (i,0,-1)$, and then the learner makes a mistake with only one hypothesis $2\ind{\{i\}} -1$ eliminated.
Therefore, the learner will make $n-1$ mistakes.
\end{example}
\vspace{-3pt}

In this work, we allow the learner to be randomized. When an $(x_t,r_t,y_t)$ is generated by the environment, the learner can randomly pick an $f_t$, and the environment does not know the realization of $f_t$ but knows the distribution where $f_t$ comes from. 
It turns out that randomization does not help much.
We prove that there exists an example in which any (possibly randomized) learner will incur $\Omega(\abs{\cH})$ mistakes.

\begin{theorem}\label{thm:x-end-online}
     There exists a feature-ball manipulation set space $\cQ$ and hypothesis class $\cH$ s.t. the mistake bound $\MB_{\perp,(x,\Delta)}\geq \abs{\cH}-1$.
     For any (randomized) algorithm $\cA$ and any $T\in \NN$,  there exists a realizable sequence of $(x_t,r_t,y_t)_{1:T}$ such that with probability at least $1-\delta$ (over randomness of $\cA$), $\cA$ makes at least $\min(\frac{T}{5\abs{\cH}\log(\abs{\cH}/\delta)}, \abs{\cH}-1)$ mistakes.
\end{theorem}
Essentially, we design an adversarial environment such that the learner has a probability of $\frac{1}{\abs{\cH}}$ of making a mistake at each round before identifying the target function $h^*$. The learner only gains information about the target function when a mistake is made.
The detailed proof is deferred to Appendix~\ref{app:x-end-online}.
Theorem~\ref{thm:x-end-online} establishes a lower bound on the mistake bound, which is $\abs{\cH}-1$. 
However, achieving this bound requires a sufficiently large number of rounds, specifically $T = \tilde \Omega (\abs{\cH}^2)$.
This raises the question of whether there exists a learning algorithm that can make $o(T)$ mistakes for any $T\leq \abs{\cH}^2$.
In Example~\ref{eg:x-after-det}, we observed that the adversary can force any deterministic learner to make $\abs{\cH}-1$ mistakes in $\abs{\cH}-1$ rounds. Consequently, no deterministic algorithm can achieve $o(T)$ mistakes.

To address this, we propose a randomized algorithm that closely resembles Algorithm~\ref{alg:halving}, with a modification in the selection of $f_t$. Instead of using line~\ref{algline:pickf}, we choose $f_t$ randomly from $\VS$ since we lack prior knowledge of $x_t$. This algorithm can be viewed as a variation of the well-known multiplicative weights method, applied exclusively during mistake rounds. For improved clarity, we present this algorithm as Algorithm~\ref{alg:mw} in Appendix~\ref{app:mw} due to space limitations.
\begin{theorem}\label{thm:mw}
    For any $T\in \NN$, Algorithm~\ref{alg:mw} will make at most $\min(\sqrt{ 4\log(\abs{\cH})T},\abs{\cH}-1)$ mistakes in expectation in $T$ rounds.
\end{theorem}
Note that the $T$-dependent upper bound in Theorem~\ref{thm:mw} matches the lower bound in Theorem~\ref{thm:x-end-online} up to a logarithmic factor when $T = \abs{\cH}^2$. This implies that approximately $\abs{\cH}^2$ rounds are needed to achieve $\abs{\cH}-1$ mistakes, which is a tight bound up to a logarithmic factor. Proof of Theorem~\ref{thm:mw} is included in
Appendix~\ref{app:mw}.

\subsubsection{Results in the PAC Learning Model}
In the PAC setting, the goal of the learner is to output a predictor $f_\out$ after the repeated interactions. 
A common class of learning algorithms, which outputs a hypothesis $f_\out \in \cH$, is called proper. 
Proper learning algorithms are a common starting point when designing algorithms for new learning problems due to their natural appeal and ability to achieve good performance, such as ERM in classic PAC learning.
However, in the current setting, we show that proper learning algorithms do not work well and require a sample size linear in $\abs{\cH}$.
The formal theorem is stated as follows and the proof is deferred to Appendix \ref{app:x-end-pac-proper}.
\begin{theorem}\label{thm:x-end-pac-proper}
     There exists a feature-ball manipulation set space $\cQ$ and hypothesis class $\cH$ s.t. $\SC_{\perp, \Delta}^\text{prop}(\epsilon, \frac{7}{8}) = \Omega(\frac{\abs{\cH}}{\epsilon})$, where $\SC_{\perp, \Delta}^\text{prop}(\epsilon,\delta)$ is the $(\epsilon,\delta)$-PAC sample complexity achievable by proper algorithms.
\end{theorem}
Theorem~\ref{thm:x-end-pac-proper} implies that any algorithm capable of achieving sample complexity sub-linear in $\abs{\cH}$ must be improper. As a result, we are inspired to devise an improper learning algorithm. Before presenting the algorithm, we introduce some notations.
For two hypotheses $h_1,h_2$, let $h_1 \vee h_2$ denote the union of them, i.e., $(h_1 \vee h_2)(x) = +1$ iff. $h_1(x)=+1$ or $h_2(x)=+1$.
Similarly, we can define the union of more than two hypotheses. 
Then for any union of $k$ hypotheses, $f= \vee_{i=1}^k h_i$, the positive region of $f$ is the union of positive regions of the $k$ hypotheses and thus, we have $d(x,f) = \min_{i\in [k]} d(x,h_i)$.
Therefore, we can decrease the distance between $f$ and any feature vector $x$ by increasing $k$. 
Based on this, we devise a new randomized algorithm with improper output, described in Algorithm~\ref{alg:end-iid-ball}.
\begin{theorem}\label{thm:x_end_pac}
For any feature-ball manipulation set space $\cQ$ and hypothesis class $\cH$, we can achieve $\SC_{\perp, (x,\Delta)}(\epsilon,\delta) = \cO(\frac{\log^2(\abs{\cH}) +\log(1/\delta)}{\epsilon}\log(\frac{1}{\delta}))$ by combining Algorithm~\ref{alg:end-iid-ball} with a standard confidence boosting technique. Note that the algorithm is improper.
\end{theorem}
\begin{algorithm}[H]\caption{}\label{alg:end-iid-ball}
    \begin{algorithmic}[1]
    \STATE Initialize the version space $\VS_0=\cH$.
    \FOR{$t=1,\ldots,T$}
    \STATE randomly pick $k_t\sim\Unif(\{1, 2, 2^2,\ldots,2^{\floor{\log_2(n_t)-1}}\})$ where $n_t = \abs{\VS_{t-1}}$;\label{algline:choosef1}
    \STATE sample $k_t$ hypotheses $h_1,\ldots,h_{k_t}$ independently and uniformly at random from $\VS_{t-1}$;
    \STATE let $f_t = \vee_{i=1}^{k_t} h_i$.\label{algline:choosef2}
    \STATE \textbf{if} {$\hat y_t \neq y_t$ and $y_t = +$} \textbf{then} $\VS_t= \VS_{t-1}\setminus \{h\in \VS_{t-1}|d(x_t,h)\geq d(x_t,f_t)\}$;\label{algline:vsupdate1}
    \STATE \textbf{else if} {$\hat y_t \neq y_t$ and $y_t = -$} \textbf{then} $\VS_t= \VS_{t-1}\setminus\{h\in \VS_{t-1}|d(x_t,h)\leq d(x_t,f_t)\}$;\label{algline:vsupdate2}
    \STATE \textbf{else} $\VS_t = \VS_{t-1}$.
    \ENDFOR
    \STATE randomly pick $\tau$ from $[T]$ and randomly sample $h_1, h_2$ from $\VS_{\tau-1}$ with replacement.\label{algline:fout1}
    \STATE \textbf{output} $h_1 \vee h_2$\label{algline:fout2}
    \end{algorithmic}
\end{algorithm}
Now we outline the high-level ideas behind Algorithm~\ref{alg:end-iid-ball}.
In correct rounds where $f_t$ makes no mistake, the predictions of all hypotheses are either correct or unknown, and thus, it is hard to determine how to make updates.
In mistake rounds, we can always update the version space similar to what was done in Strategic Halving. To achieve a poly-logarithmic dependency on $\abs{\cH}$, 
we aim to reduce a significant number of misclassifying hypotheses in mistake rounds. The maximum number we can hope to reduce is a constant fraction of the misclassifying hypotheses.
We achieve this by randomly sampling a $f_t$ (lines~\ref{algline:choosef1}-\ref{algline:choosef2}) s.t. $f_t$ makes a mistake, and 
$d(x_t,f_t)$ is greater (smaller) than the median of $d(x_t,h)$ for all misclassifying hypotheses $h$ for true negative (positive) examples.
However, due to the asymmetric nature of manipulation, which aims to be predicted as positive, the rate of decreasing misclassifications over true positives is slower than over true negatives.
To compensate for this asymmetry, we output a $f_\out = h_1 \vee h_2$ with two selected hypotheses $h_1,h_2$ (lines \ref{algline:fout1}-\ref{algline:fout2}) instead of a single one to increase the chance of positive prediction.

We prove that Algorithm~\ref{alg:end-iid-ball} can achieve small strategic loss in expectation as described in Lemma~\ref{lmm:alg-exp}. 
Then we can achieve the sample complexity in Theorem~\ref{thm:x_end_pac} by boosting Algorithm~\ref{alg:end-iid-ball} to a strong learner.
This is accomplished by running Algorithm~\ref{alg:end-iid-ball} multiple times until we obtain a good predictor.
The proofs of Lemma~\ref{lmm:alg-exp} and Theorem~\ref{thm:x_end_pac} are deferred to Appendix~\ref{app:x_end_pac}.
\begin{lemma}\label{lmm:alg-exp}
Let $S = (x_t,r_t,y_t)_{t=1}^T\sim \cD^T$ denote the i.i.d. sampled agents in $T$ rounds and let $\cA(S)$ denote the output of Algorithm~\ref{alg:end-iid-ball} interacting with $S$.
For any feature-ball manipulation set space $\cQ$ and hypothesis class $\cH$, when $T\geq \frac{320\log^2(\abs{\cH})}{\epsilon}$, we have $\EEs{\cA,S}{\err(\cA(S))}\leq \epsilon$.
\end{lemma}


\subsection{Settings $(\perp, \Delta)$ and $(\perp, \perp)$
}\label{sec:ball-delta}
\paragraph{Online learning}
As mentioned in Section~\ref{sec:model}, both the settings of $(\perp,\Delta)$ and ${(\perp,\perp)}$ are harder than the setting of $(\perp,(x,\Delta))$, all lower bounds in the setting of $(\perp,(x,\Delta))$ also hold in the former two settings.
Therefore, by Theorem~\ref{thm:x-end-online}, we have $\MB_{\perp,\perp} \geq \MB_{\perp,\Delta}\geq \MB_{\perp,(x,\Delta)} =\abs{\cH}-1$.

\paragraph{PAC learning} 
In the setting of $(\perp,\Delta)$, Algorithm~\ref{alg:end-iid-ball} is not applicable anymore since the learner lacks observation of $x_t$, making it impossible to replicate the version space update steps in lines \ref{algline:vsupdate1}-\ref{algline:vsupdate2}.
It is worth noting that both PAC learning algorithms we have discussed so far fall under a general category called conservative algorithms, depend only on information from the mistake rounds.
Specifically, an algorithm is said to be conservative if for any $t$, the predictor $f_t$ only depends on the history of mistake rounds up to $t$, i.e., $\tau<t$ with $\hat y_\tau \neq y_\tau$, and the output $f_\out$ only depends on the history of mistake rounds, i.e., $(f_t,\hat y_t, y_t, \Delta_t)_{t:\hat y_t \neq y_t}$.
Any algorithm that goes beyond this category would need to utilize the information in correct rounds. As mentioned earlier, in correct rounds, the predictions of all hypotheses are either correct or unknown, which makes it challenging to determine how to make updates.
For conservative algorithms, we present a lower bound on the sample complexity in the following theorem, which is $\tilde\Omega(\frac{\abs{\cH}}{\epsilon})$, and its proof is included in Appendix~\ref{app:delta-csv}. The optimal sample complexity in the setting $(\perp,\Delta)$ is left as an open problem.



\begin{theorem}\label{thm:delta-csv}
There exists a feature-ball manipulation set space $\cQ$ and hypothesis class $\cH$ s.t. $\SC_{\perp,\Delta}^\text{csv}(\epsilon, \frac{7}{8}) = \tilde\Omega(\frac{\abs{\cH}}{\epsilon})$, where $\SC_{\perp,\Delta}^\text{csv}(\epsilon,\delta)$ is $(\epsilon,\delta)$-PAC the sample complexity achievable by conservative algorithms.
\end{theorem}




In the setting of $(\perp,\perp)$, our problem reduces to a best arm identification problem in stochastic bandits. We prove a lower bound on the sample complexity of $\tilde\Omega(\frac{\abs{\cH}}{\epsilon})$ in Theorem~\ref{thm:x-delta-never} by reduction to stochastic linear bandits and applying the tools from information theory. The proof is deferred to Appendix~\ref{app:x-delta-never}.
\begin{theorem}\label{thm:x-delta-never}
    There exists a feature-ball manipulation set space $\cQ$ and hypothesis class $\cH$ s.t. $\SC_{\perp,\perp}(\epsilon, \frac{7}{8})= \tilde\Omega(\frac{\abs{\cH}}{\epsilon})$.
\end{theorem}

\vspace{-4pt}
\section{Non-ball Manipulations}
\vspace{-2pt}
In this section, we move on to non-ball manipulations. In ball manipulations, for any feature vector $x$, we have an ordering of hypotheses according to their distances to $x$, which helps to infer the predictions of some hypotheses without implementing them.
However, in non-ball manipulations, we don't have such structure anymore.
Therefore, even in the simplest setting of observing $x_t$ before $f_t$ and $\Delta_t$, we have the PAC sample complexity lower bounded by $\tilde\Omega(\frac{\abs{\cH}}{\epsilon})$.
\begin{theorem}\label{thm:non-ball}
    There exists a feature-manipulation set space $\cQ$ and hypothesis class $\cH$ s.t. $\SC_{x,\Delta}(\epsilon, \frac{7}{8})= \tilde\Omega(\frac{\abs{\cH}}{\epsilon})$.
\end{theorem}
The proof is deferred to Appendix~\ref{app:non-ball}. It is worth noting that in the construction of the proof, we let all agents to have their original feature vector $x_t = \bZero$ such that $x_t$ does not provide any information.
Since $(x,\Delta)$ is the simplest setting and any mistake bound can be converted to a PAC bound via standard techniques (see Section~\ref{app:mistake2pac} for more details), we have the following corollary.
\begin{corollary}\label{cor:non-ball-all}
    There exists a feature-manipulation set space $\cQ$ and hypothesis class $\cH$ s.t. for all choices of $(C,F)$, $\SC_{C,F}(\epsilon, \frac{7}{8})= \tilde\Omega(\frac{\abs{\cH}}{\epsilon})$ and $\MB_{C,F}= \tilde \Omega(\abs{\cH})$.
\end{corollary}

\vspace{-4pt}
\section{Discussion and Open Problems}
\vspace{-2pt}
In this work, we investigate the mistake bound and sample complexity of strategic classification across multiple settings.
Unlike prior work, we assume that the manipulation is personalized and unknown to the learner, which makes the strategic classification problem more challenging.
In the case of ball manipulations, when the original feature vector $x_t$ is revealed prior to choosing $f_t$, the problem exhibits a similar level of difficulty as the non-strategic setting (see Table~\ref{tab:res} for details).
However, when the original feature vector $x_t$ is not revealed beforehand, the problem becomes significantly more challenging. 
Specifically, any learner will experience a mistake bound that scales linearly with $\abs{\cH}$, and any proper learner will face sample complexity that also scales linearly with $\abs{\cH}$.
In the case of non-ball manipulations, the situation worsens. Even in the simplest setting, where the original feature is observed before choosing $f_t$ and the manipulated feature is observed afterward, any learner will encounter a linear mistake bound and sample complexity. 

Besides the question of optimal sample complexity in the setting of $(\perp,\Delta)$ as mentioned in Sec~\ref{sec:ball-delta}, there are some other fundamental open questions.

\paragraph{Combinatorial measure} 
Throughout this work, our main focus is on analyzing the dependency on the size of the hypothesis class $\abs{\cH}$ without assuming any specific structure of $\cH$. Just as VC dimension provides tight characterization for PAC learnability and Littlestone dimension characterizes online learnability, we are curious if there exists a combinatorial measure that captures the essence of strategic classification in this context. In the proofs of the most lower bounds in this work, we consider hypothesis class to be singletons, in which both the VC dimension and Littlestone dimension are $1$. Therefore, they cannot be candidates to characterize learnability in the strategic setting.

\paragraph{Agnostic setting} We primarily concentrate on the realizable setting in this work. However, investigating the sample complexity and regret bounds in the agnostic setting would be an interesting avenue for future research.

\section*{Acknowledgements}
This work was supported in part by the National Science Foundation under grant CCF-2212968, by the Simons Foundation under the Simons Collaboration on the Theory of Algorithmic Fairness, by the Defense Advanced Research Projects Agency under cooperative agreement HR00112020003. The views expressed in this work do not necessarily reflect the position or the policy of the Government and no official endorsement should be inferred. Approved for public release; distribution is unlimited.
\bibliographystyle{apalike}
\bibliography{ref}

\newpage
\appendix
\section{Technical Lemmas}
\subsection{Boosting expected guarantee to high probability guarantee}\label{subsec:boost}
Consider any (possibly randomized) PAC learning algorithm $\cA$ in strategic setting, which can output a predictor $\cA(S)$ after $T$ steps of interaction with i.i.d. agents $S\sim \cD^T$ s.t. $\EE{\err(\cA(S))}\leq \epsilon$, where the expectation is taken over both the randomness of $S$ and the randomness of algorithm.
One standard way in classic PAC learning of boosting the expected loss guarantee to high probability loss guarantee is: running $\cA$ on new data $S$ and verifying the loss of $\cA(S)$ on a validation data set; if the validation loss is low, outputting the current $\cA(S)$, and repeating this process otherwise.

We will adopt this method to boost the confidence as well. The only difference in our strategic setting is that we can not re-use validation data set as we are only allowed to interact with the data through the interaction protocol. Our boosting scheme is described in the following.
\begin{itemize}
    \item For round $r = 1, \ldots, R,$
    \begin{itemize}
        \item Run $\cA$ for $T$ steps of interactions to obtain a predictor $h_r$.
        \item Apply $h_r$ for the following $m_0$ rounds to obtain the empirical strategic loss on $m_0$, denoted as $\hat l_r =\frac{1}{m_0}\sum_{t=t_r+1}^{t_r+m_0} \loss( h_r,(x_t,r_t,y_t))$, where $t_r+1$ is the starting time of these $m_0$ rounds.
        \item Break and output $h_r$ if $\hat l_r\leq 4\epsilon$.
    \end{itemize}
    \item If for all $r\in [R]$, $\hat l_r > 4\epsilon$, output an arbitrary hypothesis.
\end{itemize}
\begin{lemma}\label{lmm:boost}
Given an algorithm $\cA$, which can output a predictor $\cA(S)$ after $T$ steps of interaction with i.i.d. agents $S\sim \cD^T$ s.t. the expected loss satisfies $\EE{\err(\cA(S))}\leq \epsilon$. Let $h_\cA$ denote the output of the above boosting scheme given algorithm $\cA$ as input. By setting $R= \log \frac{2}{\delta}$ and $m_0= \frac{3\log(4R/\delta)}{2\epsilon}$, we have $\err(h_\cA) \leq 8\epsilon$ with probability $1-\delta$. The total sample size is $R(T+m_0) =\cO(\log(\frac{1}{\delta})(T +
    \frac{\log(1/\delta)}{\epsilon}))$. 
\end{lemma}
\begin{proof}
For all $r=1,\ldots,R$, we have $\EE{\err(h_r)} \leq \epsilon$.
By Markov's inequality, we have
\begin{align*}
    \Pr(\err(h_r) >2\epsilon)\leq \frac{1}{2}\,.
\end{align*}
For any fixed $h_r$, if $\err(h_r)\geq 8\epsilon$, we will have $\hat l_r \leq 4\epsilon$ with probability $\leq e^{-m_0\epsilon}$; if $\err(h_r)\leq 2\epsilon$, we will have $\hat l_r \leq 4\epsilon$ with probability $\geq 1- e^{-2m_0\epsilon/3}$ by Chernoff bound.

Let $E$ denote the event of \{$\exists r\in [R], \err(h_r)\leq 2\epsilon$\} and $F$ denote the event of \{$\hat l_r > 4\epsilon$ for all $r\in [R]$\}.
When $F$ does not hold, our boosting will output $h_r$ for some $r\in [R]$.
\begin{align*}
    &\Pr(\err(h_\cA)>8\epsilon) \\
    \leq &  \Pr(E,\neg F)\Pr(\err(h_\cA)>8\epsilon | E, \neg F) + \Pr(E, F)+\Pr(\neg E) \\
    \leq & \sum_{r=1}^R \Pr(h_\cA = h_r,\err(h_r)>8\epsilon|E,\neg F)
    + \Pr(E, F)+\Pr(\neg E)\\
    \leq & Re^{-m_0\epsilon} + e^{-2m_0\epsilon/3} + \frac{1}{2^R}\\
   \leq  & \delta\,,
\end{align*}
by setting $R= \log \frac{2}{\delta}$ and $m_0= \frac{3\log(4R/\delta)}{2\epsilon}$.
\end{proof}

\subsection{Converting mistake bound to PAC bound}\label{app:mistake2pac}
In any setting of $(C,F)$, if there is an algorithm $\cA$ that can achieve the mistake bound of $B$, then we can convert $\cA$ to a conservative algorithm by not updating at correct rounds. The new algorithm can still achieve mistake bound of $B$ as $\cA$ still sees a legal sequence of examples.
Given any conservative online algorithm, we can convert it to a PAC learning algorithm using the standard longest survivor technique~\citep{10008965845}.
\begin{lemma}\label{lmm:mistake2pac}
    In any setting of $(C,F)$, given any conservative algorithm $\cA$ with mistake bound $B$, let algorithm $\cA'$ run $\cA$ and output the first $f_t$ which survives over $\frac{1}{\epsilon}\log(\frac{B}{\delta})$ examples.
    $\cA'$ can achieve sample complexity of 
    $\cO(\frac{B}{\epsilon}\log(\frac{B}{\delta}))$.
\end{lemma}
    \begin{proof}[Proof of Lemma~\ref{lmm:mistake2pac}]
        When the sample size $m\geq \frac{B}{\epsilon}\log(\frac{B}{\delta})$, the algorithm $\cA$ will produce at most $B$ different hypotheses and there must exist one surviving for $\frac{1}{\epsilon}\log(\frac{B}{\delta})$ rounds since $\cA$ is a conservative algorithm with at most $B$ mistakes.
        Let $h_1,\ldots,h_B$ denote these hypotheses and let $t_1,\ldots,t_B$ denote the time step they are produced.
        Then we have
        \begin{align*}
            &\Pr(f_\out = h_i \text{ and } \err(h_i) >\epsilon) = \EE{\Pr(f_\out = h_i \text{ and } \err(h_i) >\epsilon|t_i, z_{1:t_i-1})} \\
            <& \EE{(1-\epsilon)^{\frac{1}{\epsilon}\log(\frac{B}{\delta})}} = \frac{\delta}{B}\,.
        \end{align*}
        By union bound, we have
        \begin{align*}
            \Pr(\err(f_\out)>\epsilon) \leq \sum_{i=1}^B \Pr_{z_{1:T}}(f_\out = h_i \text{ and } \err(h_i) >\epsilon)<\delta.
        \end{align*}  
        We are done.
    \end{proof}

\subsection{Smooth the distribution}
\begin{lemma}\label{lmm:smooth}
    For any two data distribution $\cD_1$ and $\cD_2$, let $\cD_3 = (1-p)\cD_1+ p\cD_2$ be the mixture of them.
    For any setting of $(C,F)$ and any algorithm, let $\bP_{\cD}$ be the dynamics of $(C(x_1), f_1, y_1,\hat y_1, F(x_1,\Delta_1),\ldots, C(x_T), f_T, y_T,\hat y_T, F(x_T,\Delta_T))$ under the data distribution $\cD$.
    Then for any event $A$, we have $\abs{\bP_{\cD_3}(A) - \bP_{\cD_1}(A)}\leq 2pT$.
\end{lemma}
\begin{proof}
    Let $B$ denote the event of all $(x_t,u_t,y_t)_{t=1}^T$ being sampled from $\cD_1$. 
    Then $\bP_{\cD_3}(\neg B) \leq  pT$.
    Then 
    \begin{align*}
        \bP_{\cD_3}(A) &= \bP_{\cD_3}(A|B)\bP_{\cD_3}(B)+ \bP_{\cD_3}(A|\neg B)\bP_{\cD_3}(\neg B) \\
        & = \bP_{\cD_1}(A)\bP_{\cD_3}(B)+ \bP_{\cD_3}(A|\neg B)\bP_{\cD_3}(\neg B)\\
        & = \bP_{\cD_1}(A)(1-\bP_{\cD_3}(\neg B))+ \bP_{\cD_3}(A|\neg B)\bP_{\cD_3}(\neg B)\,.
    \end{align*}
    By re-arranging terms, we have
    \begin{align*}
        \abs{\bP_{\cD_1}(A) - \bP_{\cD_3}(A)} = \abs{\bP_{\cD_1}(A)\bP_{\cD_3}(\neg B)- \bP_{\cD_3}(A|\neg B)\bP_{\cD_3}(\neg B)}\leq 2pT\,.
    \end{align*}
\end{proof}

\section{Proof of Theorem~\ref{thm:halving}}\label{app:halving}

\begin{proof}
When a mistake occurs, there are two cases.
\begin{itemize}
    \item If $f_t$ misclassifies a true positive example $(x_t,r_t, +1)$ by negative, we know that $d(x_t, f_t)> r_t$ while the target hypothesis $h^*$ must satisfy that $d(x_t, h^*)\leq r_t$. Then any $h\in \VS$ with $d(x_t,h)\geq d(x_t,f_t)$ cannot be $h^*$ and are eliminated. Since $d(x_t,f_t)$ is the median of $\{d(x_t,h)|h\in \VS\}$, we can eliminate half of the version space.
    \item If $f_t$ misclassifies a true negative example $(x_t,r_t, -1)$ by positive, we know that $d(x_t, f_t)\leq r_t$ while the target hypothesis $h^*$ must satisfy that $d(x_t, h^*)> r_t$. Then any $h\in \VS$ with $d(x_t,h)\leq d(x_t,f_t)$ cannot be $h^*$ and are eliminated. Since $d(x_t,f_t)$ is the median of $\{d(x_t,h)|h\in \VS\}$, we can eliminate half of the version space.
\end{itemize}
Each mistake reduces the version space by half and thus, the algorithm of Strategic Halving suffers at most $\log_2(\abs{\cH})$ mistakes.
\end{proof}
\section{Proof of Theorem~\ref{thm:x-first-pac}}\label{app:x_firtst_pac}
\begin{proof}
In online learning setting, an algorithm is conservative if it updates it's current predictor only when making a mistake.
It is straightforward to check that Strategic Halving is conservative. Combined with the technique of converting mistake bound to PAC bound in Lemma~\ref{lmm:mistake2pac}, we prove Theorem~\ref{thm:x-first-pac}.
\end{proof}
\section{Proof of Theorem~\ref{thm:x-end-online}}\label{app:x-end-online}
\begin{proof}
Consider the feature space $\cX = \{\bZero, \be_1,\ldots,\be_n, 0.9\be_1,\ldots,0.9\be_n\}$, where $\be_i$'s are standard basis vectors in $\R^n$ and metric
$d(x,x') = \norm{x-x'}_2$ for all $x,x'\in \cX$.
Let the hypothesis class be a set of singletons over $\{\be_i|i\in [n]\}$, i.e., $\cH =\{2\ind{\{\be_i\}}-1|i\in [n]\}$.
We divide all possible hypotheses (not necessarily in $\cH$) into three categories:
\begin{itemize}
    \item The hypothesis $2\ind{\emptyset}-1$, which predicts all negative.
    \item For each $x\in \{\bZero, 0.9\be_1,\ldots,0.9\be_n\}$, let
    $F_{x,+}$ denote the class of hypotheses $h$ predicting $x$ as positive.   
    \item For each $i\in [n]$, let $F_i$ denote the class of hypotheses $h$ satisfying $h(x)=-1$ for all $x\in \{\bZero, 0.9\be_1,\ldots,0.9\be_n\}$ and $ h(\be_i) = +1$. And let $F_* = \cup_{i\in [n]} F_i$ denote the union of them.
\end{itemize}
Note that all hypotheses over $\cX$ fall into one of the three categories.

Now we consider a set of adversaries $E_1,\ldots,E_n$, such that the target function in the adversarial environment $E_i$ is $2\ind{\{\be_i\}}-1$.
We allow the learners to be randomized and thus, at round $t$, the learner draws an $f_t$ from a distribution $D(f_t)$ over hypotheses. 
The adversary, who only knows the distribution $D(f_t)$ but not the realization $f_t$, picks an agent $(x_t,r_t,y_t)$ in the following way. 
\begin{itemize}
    \item  Case 1: If there exists $x\in \{\bZero, 0.9\be_1,\ldots,0.9\be_n\}$ such that $\Pr_{f_t\sim D(f_t)}(f_t\in F_{x,+}) \geq c$ for some $c>0$,  then for all $j\in [n]$, the adversary $E_j$ picks $(x_t,r_t,y_t) = (x,0,-1)$. 
    Let $B^t_{1,x}$ denote the event of $f_t\in F_{x,+}$.
    \begin{itemize}
        \item In this case, the learner will make a mistake with probability $c$. Since for all $h\in \cH$, $h(\Delta(x,h,0)) = h(x) = -1$, they are all consistent with $(x,0,-1)$.
    \end{itemize}
    \item Case 2: If $\Pr_{f_t\sim D(f_t)}(f_t = 2\ind{\emptyset}-1)\geq c$, then for all $j\in [n]$, the adversary $E_j$ picks $(x_t,r_t,y_t) = (\bZero,1,+1)$. Let $B^t_2$ denote the event of $f_t = 2\ind{\emptyset}-1$.
    \begin{itemize}
        \item In this case, with probability $c$, the learner will sample a $f_t = 2\ind{\emptyset}-1$ and misclassify $(\bZero,1,+1)$.
    Since for all $h\in \cH$, $h(\Delta(\bZero,h,1)) = +1$, they are all consistent with $(\bZero,1,+1)$.
    \end{itemize}
    \item Case 3: If the above two cases do not hold, let $i_t = \argmax_{i\in [n]} \Pr(f_t(\be_i) =1|f_t\in F_*)$, $x_t = 0.9\be_{i_t}$.
    For radius and label, different adversaries set them differently. Adversary $E_{i_t}$ will set $(r_t,y_t) = (0,-1)$ while other $E_j$ for $j\neq i_t$ will set $(r_t,y_t) = (0.1, -1)$. Since Cases 1 and 2 do not hold, we have $\Pr_{f_t\sim D(f_t)}(f_t\in F_*)\geq 1-(n+2)c$. Let $B^t_3$ denote the event of $f_t\in F_*$ and $B^t_{3,i}$ denote the event of $f_t\in F_i$.
    \begin{itemize}
        \item[(a)] With probability $\Pr(B^t_{3,i_t})\geq \frac{1}{n}\Pr(B^t_3) \geq \frac{1-(n+2)c}{n}$, the learner samples a $f_t \in F_{i_t}$, and thus misclassifies $(0.9\be_{i_t},0.1,-1)$ in $E_j$ for $j\neq i_t$ but correctly classifies $(0.9\be_{i_t},0,-1)$.
        In this case, the learner observes the same feedback in all $E_j$ for $j\neq i_t$ and identifies the target function $2\ind{\{\be_{i_t}\}}-1$ in $E_{i_t}$.
        \item[(b)] If the learner samples a $f_t$ with $f_t(\be_{i_t}) =f_t(0.9\be_{i_t}) = -1$, then the learner observes $x_t = 0.9\be_{i_t}$, $y_t = -1$ and $\hat y_t =-1$ in all $E_j$ for $j\in [n]$. Therefore the learner cannot distinguish between adversaries in this case.
        \item[(c)] If the learner samples a $f_t$ with $f_t(0.9\be_{i_t}) = +1$, then the learner observes $x_t = 0.9\be_{i_t}$, $y_t = -1$ and $\hat y_t =+1$ in all $E_j$ for $j\in [n]$.
        Again, since the feedback are identical in all $E_j$ and the learner cannot distinguish between adversaries in this case.
    \end{itemize}
    \end{itemize}
For any learning algorithm $\cA$, his predictions are identical in all of  adversarial environments $\{E_j|j\in [n]\}$ before he makes a mistake in Case 3(a) in one environment $E_{i_t}$.
His predictions in the following rounds are identical in all of  adversarial environments $\{E_j|j\in [n]\}\setminus \{E_{i_t}\}$ before he makes another mistake in Case 3(a).
Suppose that we run $\cA$ in all adversarial environment of $\{E_j|j\in [n]\}$ simultaneously.
Note that once we make a mistake, the mistake must occur simultaneously in at least $n-1$ environments.
Specifically, if we make a mistake in Case 1, 2 or 3(c), such a mistake simultaneously occur in all $n$ environments.
If we make a mistake in Case 3(a), such a mistake simultaneously occur in all $n$ environments 
except $E_{i_t}$.
Since we will make a mistake with probability at least $\min(c, \frac{1-(n+2)c}{n})$ at each round, there exists one environment in $\{E_j|j\in [n]\}$ in which $\cA$ will make $n-1$ mistakes.

Now we lower bound the number of mistakes dependent on $T$.
Let $t_1,t_2,\ldots$ denote the time steps in which we makes a mistake. Let $t_0=0$ for convenience.
Now we prove that
\begin{align*}
    &\Pr(t_i > t_{i-1} + k|t_{i-1}) = \prod_{\tau = t_{i-1}+1}^{t_{i-1} + k} \Pr(\text{we don't make a mistake in round } \tau)\\
    \leq & \prod_{\tau = t_{i-1}+1}^{t_{i-1} + k} (\1{\text{Case 3 at round }\tau} (1- \frac{1-(n+2)c}{n}) + \1{\text{Case 1 or 2 at round }\tau} (1- c))\\
    \leq &
    (1-\min(\frac{1-(n+2)c}{n}, c))^k \leq (1-\frac{1}{2(n+2)})^k\,,
\end{align*}
by setting $c = \frac{1}{2(n+2)}$.
Then by letting $k = 2(n+2)\ln(n/\delta)$, we have $$\Pr(t_i > t_{i-1} + k|t_{i-1}) \leq \delta/n\,.$$
For any $T$,
\begin{align*}
    &\Pr(\text{\# of mistakes} <\min(\frac{T}{k+1}, n-1))\\
    =\leq & \Pr(\exists i\in [n-1], t_i-t_{i-1}> k) \\
    \leq& \sum_{i=1}^{n-1} \Pr(t_i-t_{i-1}> k)\leq \delta\,.
\end{align*}
Therefore, we have proved that for any $T$, with probability at least $1-\delta$, we will make at least $\min(\frac{T}{2(n+2)\ln(n/\delta)+1}, n-1)$ mistakes.
\end{proof}

\section{Proof of Theorem~\ref{thm:mw}}\label{app:mw}
\begin{algorithm}[H]\caption{MWMR (Multiplicative Weights on Mistake Rounds)}\label{alg:mw}
    \begin{algorithmic}[1]
    \STATE Initialize the version space $\VS=\cH$.
    \FOR{t=1,\ldots,T}
    \STATE Pick one hypotheses $f_t$ from $\VS$ uniformly at random.
    \IF{$\hat y_t \neq y_t$ and $y_t = +$}
    \STATE $\VS\leftarrow \VS\setminus \{h\in \VS|d(x_t,h)\geq d(x_t,f_t)\}$.
    \ELSIF{$\hat y_t \neq y_t$ and $y_t = -$}
    \STATE $\VS\leftarrow \VS \setminus \{h\in \VS|d(x_t,h)\leq d(x_t,f_t)\}$.
    \ENDIF
    \ENDFOR
    \end{algorithmic}
\end{algorithm}
\begin{proof}
    First, when the algorithm makes a mistake at round $t$, he can at least eliminate $f_t$. Therefore, the total number of mistakes will be upper bounded by $\abs{\cH}-1$.
    
    Let $p_t$ denote the fraction of hypotheses misclassifying $x_t$.
    We say a hypothesis $h$ is inconsistent with $(x_t,f_t, y_t, \hat y_t)$ iff $(d(x_t,h)\geq d(x_t,f_t)\wedge \hat y_t = - \wedge y_t = +)$ or  $(d(x_t,h)\leq d(x_t,f_t)\wedge \hat y_t = + \wedge y_t = -)$.
    Then we define the following events.
    \begin{itemize}
        \item $E_t$ denotes the event that MWMR makes a mistake at round $t$. We have $\Pr(E_t) = p_t$.
        \item $B_t$ denotes the event that at least $\frac{p_t}{2}$ fraction of hypotheses are inconsistent with $(x_t,f_t, y_t, \hat y_t)$. We have $\Pr(B_t|E_t)\geq \frac{1}{2}$.
    \end{itemize}
    Let $n =\abs{\cH}$ denote the cardinality of hypothesis class and $n_t$ denote the number of hypotheses in $\VS$ after round $t$.
    Then we have 
    \[1\leq n_T = n\cdot \prod_{t=1}^T (1- \1{E_t}\1{B_t}\frac{p_t}{2})\,.\]
    By taking logarithm of both sides, we have
    \begin{align*}
        0\leq \ln(n_T) = \ln(n) +\sum_{t=1}^T \ln(1- \1{E_t}\1{B_t}\frac{p_t}{2})\leq \ln(n) -\sum_{t=1}^T \1{E_t}\1{B_t}\frac{p_t}{2}\,,
    \end{align*}
    where the last inequality adopts $\ln(1-x)\leq -x$ for $x\in [0,1)$.
    Then by taking expectation of both sides,
    we have
    \begin{equation*}
        0\leq \ln(n) - \sum_{t=1}^T \Pr(E_t \wedge B_t)\frac{p_t}{2}\,.
    \end{equation*}
    Since $\Pr(E_t) = p_t$ and $\Pr(B_t|E_t)\geq \frac{1}{2}$, then we have
    \begin{align*}
        \frac{1}{4}\sum_{t=1}^T p_t^2\leq \ln(n)\,.
    \end{align*}
    Then we have the expected number of mistakes $\EE{\cM_\MW(T)}$ as
    \begin{equation*}
        \EE{\cM_\MW(T)} = \sum_{t=1}^T p_t \leq \sqrt{\sum_{t=1}^T p_t^2 } \cdot \sqrt{T}\leq \sqrt{4\ln(n)T}\,,
    \end{equation*}
    where the first inequality applies Cauchy-Schwarz inequality.
\end{proof}
\section{Proof of Theorem~\ref{thm:x-end-pac-proper}}\label{app:x-end-pac-proper}
\begin{proof}
\textbf{Construction of $\cQ,\cH$ and a set of realizable distributions}
\begin{itemize}
    \item Let feature space $\cX = \{\bZero, \be_1,\ldots, \be_n\}\cup X_0$, where $X_0 = \{\frac{\sigma(0,1,\ldots,n-1)}{z}|\sigma\in \cS_n\}$ with $z = \frac{\sqrt{1^2+\ldots+(n-1)^2}}{\alpha}$ for some small $\alpha=0.1$.
    Here $\cS_n$ is the set of all permutations over $n$ elements.
    So $X_0$ is the set of points whose coordinates are a permutation of $\{0,1/z,\ldots,(n-1)/z\}$ and all points in $X_0$  have the  $\ell_2$ norm equal to $\alpha$.
    Define a metric $d$ by letting $d(x_1,x_2) = \norm{x_1-x_2}_2$ for all $x_1,x_2\in \cX$.
    Then for any $x\in X_0$ and $i\in [n]$, $d(x,\be_i) = \norm{x-\be_i}_2 = \sqrt{(x_i-1)^2 +\sum_{j\neq i} x_j^2} = \sqrt{1+\sum_{j=1}^n x_j^2 - 2x_i} = \sqrt{1+\alpha^2-2x_i}$.
    Note that we consider space $(\cX,d)$ rather than $(\R^n, \norm{\cdot}_2)$.
    \item Let the hypothesis class be a set of singletons over $\{\be_i|i\in [n]\}$, i.e., $\cH = \{2\ind{\{\be_i\}}-1|i\in [n]\}$.
    \item We now define a collection of distributions $\{\cD_i|i\in [n]\}$ in which $\cD_i$ is realized by $2\ind{\{\be_i\}}-1$.
    For any $i\in [n]$, $\cD_i$ puts probability mass $1-3n\epsilon$ on $(\bZero, 0, -1)$.
    For the remaining $3n\epsilon$ probability mass, $\cD_i$ picks $x$ uniformly at random from $X_0$ and label it as positive.
    If $x_{i} =0$, set radius $r(x) = r_u := \sqrt{1+\alpha^2}$; otherwise, set radius $r(x) = r_l := \sqrt{1+\alpha^2 -2\cdot \frac{1}{z})}$.
    Hence, $X_0$ are all labeled as positive.
    For $j\neq i$, $h_j = 2\ind{\{\be_j\}}-1$ labels $\{x\in X_0|x_j = 0\}$ negative since $r(x) = r_l$ and $d(x,h_j) = r_u> r(x)$.
    Therefore, $\err(h_j) = \frac{1}{n} \cdot 3n\epsilon = 3\epsilon$.
    To output $f_\out \in \cH$, we must identify the true target function.
\end{itemize}

\textbf{Information gain from different choices of $f_t$} 
Let $h^* = 2\ind{\{\be_{i^*}\}}-1$ denote the target function.
Since $(\bZero, 0, -1)$ is realized by all hypotheses, we can only gain information about the target function when $x_t\in X_0$.
For any $x_t \in X_0$, if $d(x_t,f_t)\leq r_l$ or $d(x_t,f_t)> r_u$, we cannot learn anything about the target function. 
In particular, if $d(x_t,f_t)\leq r_l$, the learner will observe $x_t\sim \Unif(X_0)$, $y_t = +1$, $\hat y_t = +1$ in all $\{\cD_i|i\in [n]\}$.
If $d(x_t,f_t)> r_u$, the learner will observe $x_t\sim \Unif(X_0)$, $y_t = +1$, $\hat y_t = -1$ in all $\{\cD_i|i\in [n]\}$.
Therefore, we cannot obtain any information about the target function.

Now for any $x_t \in X_0$, with the $i_t$-th coordinate being $0$, we enumerate the distance between $x$ and $x'$ for all $x'\in \cX$.
\begin{itemize}
    \item For all $x'\in X_0$, $d(x,x')\leq \norm{x}+\norm{x'}\leq 2\alpha < r_l$;
    \item For all $j\neq i_t$, $d(x,\be_j) = \sqrt{1+\alpha^2 - 2 x_j}\leq r_l$;
    \item $d(x,\be_{i_t}) = r_u$;
    \item $d(x,\bZero) =\alpha <r_l$.
\end{itemize}
Only $f_t = 2\ind{\{\be_{i_t}\}}-1$ satisfies that $r_l<d(x_t,f_t)\leq r_u$ and thus, we can only obtain information when $f_t = 2\ind{\{\be_{i_t}\}}-1$.
And the only information we learn is whether $i_t = i^*$ because if $i_t\neq i^*$, no matter which $i^*$ is, our observation is identical. If $i_t\neq i^*$, we can eliminate $2\ind{\{\be_{i_t}\}}-1$.

\textbf{Sample size analysis}
For any algorithm $\cA$, his predictions are identical in all environments $\{\cD_i|i\in [n]\}$ before a round $t$ in which $f_t = 2\ind{\{\be_{i_t}\}}-1$. Then either he learns $i_t$ in $\cD_{i_t}$ or he eliminates $2\ind{\{\be_{i_t}\}}-1$ and continues to perform the same in the other environments $\{\cD_i|i\neq i_t\}$.
Suppose that we run $\cA$ in all stochastic environments $\{\cD_i|i\in [n]\}$ simultaneously. When we identify $i_t$ in environment $\cD_{i_t}$, we terminate $\cA$ in $\cD_{i_t}$. Consider a good algorithm $\cA$ which can identify $i$ in $\cD_i$ with probability $\frac{7}{8}$ after $T$ rounds of interaction for each $i\in [n]$, that is,
\begin{align}
    \Pr_{\cD_i,\cA}(i_\out \neq i)\leq \frac{1}{8},\forall i\in[n]\,.\label{eq:good4all}
\end{align}
Therefore, we have
\begin{align}
    \sum_{i\in [n]}\Pr_{\cD_i,\cA}(i_\out \neq i)\leq \frac{n}{8}\,.\label{eq:good4sum}
\end{align}


Let $n_T$ denote the number of environments that have been terminated by the end of round $T$.
Let $B_t$ denote the event of $x_t$ being in $X_0$ and $C_t$ denote the event of $f_t=2\ind{\{\be_{i_t}\}}-1$.
Then we have $\Pr(B_t) = 3n\epsilon$ and $\Pr(C_t|B_t) = \frac{1}{n}$, and thus $\Pr(B_t \wedge C_t) = 3n\epsilon \cdot \frac{1}{n}$.
Since at each round, we can eliminate one environment only when $B_t\wedge C_t$ is true, then we have
\begin{align*}
    \EE{n_T} \leq \EE{\sum_{t=1}^T\1{B_t\wedge C_t}} = T\cdot 3n\epsilon \cdot \frac{1}{n} = 3\epsilon T\,.
\end{align*}
Therefore, by setting $T = \frac{\floor{\frac{n}{2}}-1}{6\epsilon}$ and Markov's inequality, we have
\begin{align*}
    \Pr(n_T\geq \floor{\frac{n}{2}}-1)\leq \frac{3\epsilon T}{\floor{\frac{n}{2}}-1} = \frac{1}{2}\,.
\end{align*}
When there are $\ceil{\frac{n}{2}}+1$ environments remaining, the algorithm has to pick one $i_\out$, which fails in at least $\ceil{\frac{n}{2}}$ of the environments.
Then we have
\begin{align*}
    \sum_{i\in [n]}\Pr_{\cD_i,\cA}(i_\out \neq i)\geq \ceil{\frac{n}{2}}\Pr(n_T\leq \floor{\frac{n}{2}}-1)\geq \frac{n}{4}\,,
\end{align*}
which conflicts with Eq~\eqref{eq:good4sum}.
Therefore, for any algorithm $\cA$, to achieve Eq~\eqref{eq:good4all}, it requires $T\geq \frac{\floor{\frac{n}{2}}-1}{6\epsilon}$.
\end{proof}
\section{Proof of Theorem~\ref{thm:x_end_pac}}\label{app:x_end_pac}
Given Lemma~\ref{lmm:alg-exp}, we can upper bound the expected strategic loss, then we can boost the confidence of the algorithm through the scheme in Section~\ref{subsec:boost}.
Theorem~\ref{thm:x_end_pac} follows by combining Lemma~\ref{lmm:alg-exp} and Lemma~\ref{lmm:boost}.
Now we only need to prove Lemma~\ref{lmm:alg-exp}.
\begin{proof}[Proof of Lemma~\ref{lmm:alg-exp}]
    For any set of hypotheses $H$, for every $z = (x,r,y)$, we define 
    \begin{align*}
        \kappa_p(H, z) := \begin{cases}
            \abs{\{h\in H|h(\Delta(x,h,r)) = -\}} & \text{if } y = +\,,\\
            0 & \text{otherwise.}
        \end{cases}
    \end{align*}
    So $\kappa_p(H, z)$ is the number of hypotheses mislabeling $z$ for positive $z$'s and $0$ for negative $z$'s.
    Similarly, we define $\kappa_n$ as follows,
    \begin{align*}
        \kappa_n(H, z) := \begin{cases}
            \abs{\{h\in H|h(\Delta(x,h,r)) = +\}} & \text{if } y = -\,,\\
            0 & \text{otherwise.}
        \end{cases}
    \end{align*}
    So $\kappa_n(H, z)$ is the number of hypotheses mislabeling $z$ for negative $z$'s and $0$ for positive $z$'s.

In the following, we divide the proof into two parts. First, recall that in Algorithm~\ref{alg:end-iid-ball}, the output is constructed by randomly sampling two hypotheses with replacement and taking the union of them. We represent the loss of such a random predictor using $\kappa_p(H, z)$ and $\kappa_n(H, z)$ defined above.
Then we show that whenever the algorithm makes a mistake, with some probability, we can reduce $\frac{\kappa_p(\VS_{t-1}, z_t)}{2}$ or $\frac{\kappa_n(\VS_{t-1}, z_t)}{2}$ hypotheses and utilize this to provide a guarantee on the loss of the final output.

\paragraph{Upper bounds on the strategic loss} 
For any hypothesis $h$, let $\fpr(h)$ and $\fnr(h)$ denote the false positive rate and false negative rate of $h$ respectively.
Let $p_+$ denote the probability of drawing a positive sample from $\cD$, i.e., $\Pr_{(x,r,y)\sim \cD}(y=+)$ and $p_-$ denote the probability of drawing a negative sample from $\cD$.
Let $\cD_+$ and $\cD_-$ denote the data distribution conditional on that the label is positive and that the label is negative respectively.
Given any set of hypotheses $H$, we define a random predictor $R2(H) = h_1 \vee h_2$ with $h_1, h_2$ randomly picked from $H$ with replacement.
    For a true positive $z$, $R2(H)$ will misclassify it with probability $\frac{\kappa_p(H,z)^2}{\abs{H}^2}$.
    Then we can find that the false negative rate of $R2(H)$ is
    \begin{align*}
        \fnr(R2(H)) = \EEs{z=(x,r,+)\sim \cD_+}{\Pr(R2(H)(x) = -)} = \EEs{z=(x,r,+)\sim \cD_+}{\frac{\kappa_p(H,z)^2}{\abs{H}^2}}\,.
    \end{align*}
    Similarly, for a true negative $z$, $R2(H)$ will misclassify it with probability $1- (1-\frac{\kappa_n(H,z)}{\abs{H}})^2 \leq \frac{2\kappa_n(H,z)}{\abs{H}}$.
    Then the false positive rate of $R2(H)$ is
    \begin{align*}
        \fpr(R2(H)) = \EEs{z=(x,r,-)\sim \cD_-}{\Pr(R2(H)(x) = +)} \leq \EEs{z=(x,r,-)\sim \cD_+}{\frac{2\kappa_n(H,z)}{\abs{H}}}\,.
    \end{align*}
    Hence the loss of $R2(H)$ is
    \begin{align}
        \err(R2(H)) &\leq p_+\EEs{z\sim \cD_+}{\frac{\kappa_p(H,z)^2}{\abs{H}^2}} + p_- \EEs{z\sim \cD_+}{\frac{2\kappa_n(H,z)}{\abs{H}}}\nonumber\\
        &= \EEs{z\sim \cD}{\frac{\kappa_p(H,z)^2}{\abs{H}^2} + 2\frac{\kappa_n(H,z)}{\abs{H}}}\,,\label{eq:loss-rep}
    \end{align}
    where the last equality holds since $\kappa_p(H,z) = 0$ for true negatives and $\kappa_n(H,z) = 0$ for true positives.

    \paragraph{Loss analysis} In each round, the data $z_t = (x_t,r_t,y_t)$ is sampled from $\cD$.
    When the label $y_t$ is positive, if the drawn $f_t$ satisfying that 1) $f_t(\Delta(x_t,f_t,r_t)) = -$ and 2) $d(x_t,f_t)\leq \textrm{median}(\{d(x_t,h)|h\in \VS_{t-1}, h(\Delta(x_t,h,r_t))=-\})$, then we are able to remove 
    $\frac{\kappa_p(\VS_{t-1},z_t)}{2}$ hypotheses from the version space.
    Let $E_{p,t}$ denote the event of $f_t$ satisfying the conditions 1) and 2).
    With probability $\frac{1}{\floor{\log_2(n_t)}}$, we sample $k_t=1$.
    Then we sample an $f_t\sim \Unif(\VS_{t-1})$.
    With probability $\frac{\kappa_p(\VS_{t-1},z_t)}{2n_t}$, the sampled $f_t$ satisfies the two conditions.
    So we have
    \begin{equation}
        \Pr(E_{p,t}|z_t,\VS_{t-1})\geq \frac{1}{\log_2(n_t)}\frac{\kappa_p(\VS_{t-1},z_t)}{2n_t}\,.\label{eq:eventp}
    \end{equation}
     

    The case of $y_t$ being negative is similar to the positive case.
    Let $E_{n,t}$ denote the event of $f_t$ satisfying that 1) $f_t(\Delta(x_t,f_t,r_t)) = +$ and 2) $d(x_t,f_t)\geq \textrm{median}(\{d(x_t,h)|h\in \VS_{t-1}, h(\Delta(x_t,h,r_t))=+\})$.
    If $\kappa_n(\VS_{t-1},z_t)\geq \frac{n_t}{2}$, then with probability $\frac{1}{\floor{\log_2(n_t)}}$, we sample $k_t = 1$.
    Then with probability greater than $\frac{1}{4}$ we will sample an $f_t$ satisfying that 1) $f_t(\Delta(x_t,f_t,r_t)) = +$ and 2) $d(x_t,f_t)\geq \textrm{median}(\{d(x_t,h)|h\in \VS_{t-1}, h(\Delta(x_t,h,r_t))=+\})$.
    If $\kappa_n(\VS_{t-1},z_t)< \frac{n_t}{2}$, then
    with probability $\frac{1}{\floor{\log_2(n_t)}}$, we sampled a $k_t$ satisfying 
    \[\frac{n_t}{4\kappa_n(\VS_{t-1},z_t)} <k_t \leq \frac{n_t}{2\kappa_n(\VS_{t-1},z_t)}\,.\]
    Then we randomly sample $k_t$ hypotheses and the expected number of sampled hypotheses which mislabel $z_t$ is $k_t\cdot \frac{\kappa_n(\VS_{t-1},z_t)}{n_t} \in (\frac{1}{4},\frac{1}{2}]$.
    Let $g_t$ (given the above fixed $k_t$) denote the number of sampled hypotheses which mislabel $x_t$ and we have $\EE{g_t} \in (\frac{1}{4},\frac{1}{2}]$.
    When $g_t>0$, $f_t$ will misclassify $z_t$ by positive.
    We have
    \[\Pr(g_t = 0) = (1-\frac{\kappa_n(\VS_{t-1},z_t)}{n_t})^{k_t} < (1-\frac{\kappa_n(\VS_{t-1},z_t)}{n_t})^{\frac{n_t}{4\kappa_n(\VS_{t-1},z_t)}}\leq e^{-1/4}\leq 0.78\]
    and by Markov's inequality, we have
    \[\Pr(g_t \geq 3) \leq \frac{\EE{g_t}}{3} \leq \frac{1}{6}\leq 0.17\,.\]
    Thus $\Pr(g_t \in \{1,2\}) \geq 0.05$.
    Conditional on $g_t$ is either $1$ or $2$, with probability $\geq \frac{1}{4}$, all of these $g_t$ hypotheses $h'$ satisfies $d(x_t,h')\geq \textrm{median}(\{d(x_t,h)|h\in \VS_{t-1}, h(\Delta(x_t,h,r_t))=+\})$, which implies that $d(x_t,f_t)\geq \textrm{median}(\{d(x_t,h)|h\in \VS_{t-1}, h(\Delta(x_t,h,r_t))=+\})$.
    Therefore, we have
    \begin{equation}
        \Pr(E_{n,t}|z_t,,\VS_{t-1})\geq \frac{1}{80\log_2(n_t)}\,.\label{eq:eventn}
    \end{equation}
    
    Let $v_t$ denote the fraction of hypotheses we eliminated at round $t$, i.e., $v_t = 1 -\frac{n_{t+1}}{n_t}$ . Then we have
    \begin{equation}
        v_t\geq \1{E_{p,t}}\frac{\kappa_p(\VS_{t-1},z_t)}{2n_t} + \1{E_{n,t}}\frac{\kappa_n(\VS_{t-1},z_t)}{2n_t}\,.\label{eq:removefrac}
    \end{equation}
    Since $n_{t+1} = n_t(1-v_t)$, we have
    \[1 \leq n_{T+1} = n \prod_{t=1}^T (1-v_t)\,.\]
    By taking logarithm of both sides, we have
    \begin{align*}
        0 \leq \ln n_{T+1} = \ln n +\sum_{t=1}^T \ln(1-v_t)\leq \ln n -\sum_{t=1}^T v_t\,,
    \end{align*}
    where we use $\ln(1-x)\leq -x$ for $x\in [0,1)$ in the last inequality.
    By re-arranging terms, we have
    \[\sum_{t=1}^T v_t \leq \ln n\,.\]
    Combined with Eq~\eqref{eq:removefrac}, we have
    \begin{align*}
        \sum_{t=1}^T \1{E_{p,t}}\frac{\kappa_p(\VS_{t-1},z_t)}{2n_t} + \1{E_{n,t}}\frac{\kappa_n(\VS_{t-1},z_t)}{2n_t}\leq \ln n\,.
    \end{align*}
    By taking expectation w.r.t. the randomness of $f_{1:T}$ and dataset $S = z_{1:T}$ on both sides, we have
    \[\sum_{t=1}^T\EEs{f_{1:T},z_{1:T}}{\1{E_{p,t}}\frac{\kappa_p(\VS_{t-1},z_t)}{2n_t} + \1{E_{n,t}}\frac{\kappa_n(\VS_{t-1},z_t)}{2n_t}} \leq \ln n\,.\]
    Since the $t$-th term does not depend on $f_{t+1:T},z_{t+1:T}$ and $\VS_{t-1}$ is determined by $z_{1:t-1}$ and $f_{1:t-1}$, the $t$-th term becomes
    \begin{align}
        &\EEs{f_{1:t},z_{1:t}}{\1{E_{p,t}}\frac{\kappa_p(\VS_{t-1},z_t)}{2n_t} + \1{E_{n,t}}\frac{\kappa_n(\VS_{t-1},z_t)}{2n_t}} \nonumber\\
    = &\EEs{f_{1:t-1},z_{1:t}}{\EEs{f_t}{\1{E_{p,t}}\frac{\kappa_p(\VS_{t-1},z_t)}{2n_t} + \1{E_{n,t}}\frac{\kappa_n(\VS_{t-1},z_t)}{2n_t}| f_{1:t-1},z_{1:t}}}\nonumber\\
    =& \EEs{f_{1:t-1},z_{1:t}}{\EEs{f_t}{\1{E_{p,t}}|f_{1:t-1},z_{1:t}}\frac{\kappa_p(\VS_{t-1},z_t)}{2n_t} + \EEs{f_t}{\1{E_{n,t}}|f_{1:t-1},z_{1:t}}\frac{\kappa_n(\VS_{t-1},z_t)}{2n_t}}\label{eq:vs-ndep-f}\\
    \geq &\EEs{f_{1:t-1},z_{1:t}}{\frac{1}{\log_2(n_t)}\frac{\kappa_p^2(\VS_{t-1},z_t)}{4n_t^2} + \frac{1}{80\log_2(n_t)}\frac{\kappa_n(\VS_{t-1},z_t)}{2n_t}}\label{eq:prob-good}\,,
    \end{align}
    where Eq~\eqref{eq:vs-ndep-f} holds due to that $\VS_{t-1}$ is determined by $f_{1:t-1}, z_{1:t-1}$ and does not depend on $f_t$ and Eq~\eqref{eq:prob-good} holds since $\Pr_{f_t}(E_{p,t}|f_{1:t-1},z_{1:t}) = \Pr_{f_t}(E_{p,t}|\VS_{t-1},z_t) \geq \frac{1}{\log_2(n_t)}\frac{\kappa_p(\VS_{t-1},z_t)}{2n_t}$ by Eq~\eqref{eq:eventp} and $\Pr_{f_t}(E_{n,t}|f_{1:t-1},z_{1:t})=\Pr_{f_t}(E_{n,t}|\VS_{t-1},z_t) \geq \frac{1}{80\log_2(n_t)}$ by Eq~\eqref{eq:eventn}.
    Thus, we have
    \[\sum_{t=1}^T \EEs{f_{1:t-1},z_{1:t}}{\frac{1}{\log_2(n_t)}\frac{\kappa_p^2(\VS_{t-1},z_t)}{4n_t^2} + \frac{1}{80\log_2(n_t)}\frac{\kappa_n(\VS_{t-1},z_t)}{2n_t}} \leq \ln n\,.\]
    Since $z_t \sim \cD$ and $z_t$ is independent of $z_{1:t-1}$ and $f_{1:t-1}$, thus, 
    we have the $t$-th term on the LHS being
    \begin{align*}
        &\EEs{f_{1:t-1},z_{1:t}}{\frac{1}{\log_2(n_t)}\frac{\kappa_p^2(\VS_{t-1},z_t)}{4n_t^2} + \frac{1}{80\log_2(n_t)}\frac{\kappa_n(\VS_{t-1},z_t)}{2n_t} }\\
    =& \EEs{f_{1:t-1},z_{1:t-1}}{\EEs{z_t \sim \cD}{\frac{1}{\log_2(n_t)}\frac{\kappa_p^2(\VS_{t-1},z_t)}{4n_t^2} + \frac{1}{80\log_2(n_t)}\frac{\kappa_n(\VS_{t-1},z_t)}{2n_t}}}\\
    \geq& \frac{1}{320\log_2(n)}\EEs{f_{1:t-1},z_{1:t-1}}{\EEs{z\sim \cD}{\frac{\kappa_p^2(\VS_{t-1},z)}{n_t^2} + \frac{2\kappa_n(\VS_{t-1},z)}{n_t}}}\\
    \geq& \frac{1}{320\log_2(n)}\EEs{f_{1:t-1},z_{1:t-1}}{\err(R2(\VS_{t-1}))}\,,
    \end{align*}
    where the last inequality adopts Eq~\eqref{eq:loss-rep}.
    By summing them up and re-arranging terms,
    we have
    \[\EEs{f_{1:T},z_{1:T}}{\frac{1}{T}\sum_{t=1}^T \err(R2(\VS_{t-1}))}=\frac{1}{T}\sum_{t=1}^T \EEs{f_{1:t-1},z_{1:t-1}}{\err(R2(\VS_{t-1}))} \leq \frac{320\log_2(n)\ln(n)}{T}\,.\]
    For the output of Algorithm~\ref{alg:end-iid-ball}, which randomly picks $\tau$ from $[T]$, randomly samples $h_1, h_2$ from $\VS_{\tau-1}$ with replacement and outputs $h_1\vee h_2$, the expected loss is
    \begin{align*}
        \EE{\err(\cA(S))}= &\EEs{S,f_{1:T}}{\frac{1}{T}\sum_{t=1}^T \EEs{h_1,h_2\sim \Unif(\VS_{t-1})}{\err(h_1\vee h_2)}} \\
        =& \EEs{S,f_{1:T}}{\frac{1}{T}\sum_{t=1}^T \err(R2(\VS_{t-1}))}\\
        \leq& \frac{320\log_2(n)\ln(n)}{T} \leq \epsilon\,,
    \end{align*}
    when $T\geq \frac{320\log_2(n)\ln(n)}{\epsilon}$.
\end{proof}  
\paragraph{Post proof discussion of Lemma~\ref{lmm:alg-exp}}
\begin{itemize}
    \item Upon first inspection, readers might perceive a resemblance between the proof of the loss analysis section and the standard proof of converting regret bound to error bound.This standard proof converts a regret guarantee on $f_{1:T}$ to an error guarantee of $\frac{1}{T}\sum_{t=1}^T f_t$.
    However, in this proof, the predictor employed in each round is $f_t$, while the output is an average over $R2(\VS_{t-1})$ for all $t\in [T]$. Our algorithm does not provide a regret guarantee on $f_{1:T}$.
    \item Please note that our analysis exhibits asymmetry regarding losses on true positives and true negatives. 
    Specifically, the probability of identifying and reducing half of the misclassifying hypotheses on true positives, denoted as $\Pr(E_{p,t}|z_t,\VS_{t-1})$ (Eq~\eqref{eq:eventp}), is lower than the corresponding probability for true negatives, $\Pr(E_{n,t}|z_t,\VS_{t-1})$ (Eq~\eqref{eq:eventn}). This discrepancy arises due to the different levels of difficulty in detecting misclassifying hypotheses.
    For example, if there is exactly one hypothesis $h$ misclassifying a true positive $z_t=(x_t,r_t,y_t)$, it is very hard to detect this $h$.
    We must select an $f_t$ satisfying that $d(x_t,f_t)> d(x_t,h')$ for all $h'\in \cH\setminus \{h\}$ (hence $f_t$ will make a mistake), and that $d(x_t,f_t)\leq d(x_t,h)$ (so that we will know $h$ misclassifies $z_t$). Algorithm~\ref{alg:end-iid-ball} controls the distance $d(x_t,f_t)$ through $k_t$, which is the number of hypotheses in the union. In this case, we can only detect $h$ when $k_t =1$ and $f_t = h$, which occurs with probability $\frac{1}{n_t \log(n_t)}$.
    
    However, if there is exactly one hypothesis $h$ misclassifying a true negative $z_t=(x_t,r_t,y_t)$, we have that $d(x_t,h) = \min_{h'\in \cH} d(x_t,h')$.
    Then by setting $f_t = \vee_{h\in \cH} h$, which will makes a mistake and tells us $h$ is a misclassifying hypothesis. Our algorithm will pick such an $f_t$ with probability $\frac{1}{\log(n_t)}$.
\end{itemize}
\section{Proof of Theorem~\ref{thm:delta-csv}}\label{app:delta-csv}
\begin{proof} 
We will prove Theorem~\ref{thm:delta-csv} by constructing an instance of $\cQ$ and $\cH$ and showing that for any conservative learning algorithm, there exists a realizable data distribution s.t. achieving $\epsilon$ loss requires at least $\tilde \Omega(\frac{\abs{\cH}}{\epsilon})$ samples.
\paragraph{Construction of $\cQ$, $\cH$ and a set of realizable distributions} 
    \begin{itemize}
        \item Let the input metric space $(\cX,d)$ be constructed in the following way. Consider the feature space $\cX = \{\be_1,\ldots, \be_n\}\cup X_0$,  where $X_0 = \{\frac{\sigma(0,1,\ldots,n-1)}{z}|\sigma\in \cS_n\}$ with $z = \frac{\sqrt{1^2+\ldots+(n-1)^2}}{\alpha}$ for some small $\alpha=0.1$. Here $\cS_n$ is the set of all permutations over $n$ elements. 
        So $X_0$ is the set of points whose coordinates are a permutation of $\{0,1/z,\ldots,(n-1)/z\}$ and all points in $X_0$  have the  $\ell_2$ norm equal to $\alpha$.
        We define the metric $d$ by restricting $\ell_2$ distance to $\cX$, i.e., $d(x_1,x_2) = \norm{x_1-x_2}_2$ for all $x_1,x_2\in \cX$. 
        Then we have that for any $x\in X_0$ and $i\in [n]$, the distance between $x$ and $\be_i$ is 
        $$d(x,\be_i) = \norm{x-\be_i}_2 = \sqrt{(x_i-1)^2 +\sum_{j\neq i} x_j^2} = \sqrt{1+\sum_{j=1}^n x_j^2 - 2x_i} = \sqrt{1+\alpha^2-2x_i}\,,$$
        which is greater than $\sqrt{1+\alpha^2-2\alpha}>0.8>2\alpha$.
        For any two points $x,x'\in X_0$, $d(x,x')\leq 2\alpha$ by triangle inequality.
        
        \item Let the hypothesis class be a set of singletons over $\{\be_i|i\in [n]\}$, i.e., $\cH = \{2\ind{\{\be_i\}}-1|i\in [n]\}$.
        \item We now define a collection of distributions $\{\cD_i|i\in[n]\}$ in which $\cD_i$ is realized by $2\ind{\{\be_i\}}-1$. 
        For any $i\in [n]$, we define $\cD_i$ in the following way.
        Let the marginal distribution $\cD_\cX$ over $\cX$ be uniform over $X_0$.
        For any $x$, the label $y$ is $+$ with probability $1-6\epsilon$ and $-$ with probability $6\epsilon$, i.e., $\cD(y|x)= \Rad(1-6\epsilon)$.
        Note that the marginal distribution $\cD_{\cX\times\cY} = \Unif(X_0)\times \Rad(1-6\epsilon)$ is identical for any distribution in $\{\cD_i|i\in[n]\}$ and does not depend on $i$.
        
        If the label is positive $y=+$, then let the radius $r = 2$. If the label is negative $y=-$, then let $r= \sqrt{1+\alpha^2-2(x_{i} +\frac{1}{z})}$, which guarantees that $x$ can be manipulated to $\be_j$ iff $d(x,\be_j)< d(x,\be_{i})$ for all $j\in [n]$. 
        Since $x_i\leq \alpha$ and $\frac{1}{z}\leq \alpha$, we have $\sqrt{1+\alpha^2-2(x_{i} +\frac{1}{z})} \geq \sqrt{1-4\alpha}>2\alpha$. Therefore, for both positive and negative examples, we have radius $r$ strictly greater than $2\alpha$ in both cases.
    \end{itemize}

    \paragraph{Randomization and improperness of the output $f_\out$ do not help}
    Note that algorithms are allowed to output a randomized $f_\out$ and to output $f_\out\notin \cH$.
    We will show that randomization and improperness of $f_\out$ don't make the problem easier.
    That is, supposing that the data distribution is $\cD_{i^*}$ for some $i^*\in [n]$, finding a (possibly randomized and improper) $f_\out$ is not easier than identifying $i^*$.
    Since our feature space $\cX$ is finite, we can enumerate all hypotheses not equal to $2\ind{\{\be_{i^*}\}}-1$ and calculate their strategic population loss as follows.
    \begin{itemize}
        \item $2\ind{\emptyset}-1$ predicts all negative and thus
    $\err(2\ind{\emptyset}-1) = 1-6\epsilon$;
    \item For any $a\subset \cX$ s.t. $a\cap X_0\neq \emptyset$, 
    $2\ind{a}-1$ will predict any point drawn from $\cD_{i^*}$ as positive (since all points have radius greater than $2\alpha$ and the distance between any two points in $X_0$ is smaller than $2\alpha$) and thus
    $\err(2\ind{a}-1) = 6\epsilon$;
    \item For any $a\subset \{\be_1,\ldots,\be_n\}$ satisfying that $\exists i\neq i^*$, $\be_i\in a$, we have  $\err(2\ind{a}-1)\geq 3\epsilon$. This is due to that when $y=-$, $x$ is chosen from $\Unif(X_0)$ and the probability of $d(x,\be_i)<d(x,\be_{i^*})$ is $\frac{1}{2}$.
    When $d(x,\be_i)<d(x,\be_{i^*})$, $2\ind{a}-1$ will predict $x$ as positive.
    \end{itemize}
    Under distribution $\cD_{i^*}$, if we are able to find a (possibly randomized) $f_\out$ with strategic loss of $\err(f_\out)\leq \epsilon$, then we have $\err(f_\out) = \EEs{h\sim f_\out}{\err(h)}\geq \Pr_{h\sim f_\out}(h\neq 2\ind{\{\be_{i^*}\}}-1) \cdot 3\epsilon$.
    Thus, $\Pr_{h\sim f_\out}(h= 2\ind{\{\be_{i^*}\}}-1)\geq \frac{2}{3}$.
    Hence, if we are able to find a (possibly randomized) $f_\out$ with $\epsilon$ error, then we are able to identify $i^*$ by checking which realization of $f_\out$ has probability greater than $\frac{2}{3}$.
    In the following, we will focus on the sample complexity to identify $i^*$.
    Let $i_\out$ denote the algorithm's answer to question ``what is $i^*$?''.
    \paragraph{Conservative algorithms}
    When running a conservative algorithm, the rule of choosing $f_t$ at round $t$ and choosing the final output $f_\out$ does not depend on the correct rounds, i.e. $\{\tau \in [T]|\hat y_\tau = y_\tau\}$.
    Let's define
    \begin{align}
        \Delta_t' = \begin{cases}
            \Delta_t & \text{ if } \hat y_t \neq  y_t\\
            \perp & \text{ if } \hat y_t = y_t\,,
        \end{cases}\label{eq:deltaprime}
    \end{align}
    where $\perp$ is just a symbol representing ``no information''.
    Then for any conservative algorithm, the selected predictor $f_t$ is determined by $(f_\tau,\hat y_\tau, y_\tau, \Delta'_\tau)$ for $\tau<t$ and the final output $f_\out$ is determined by $(f_t,\hat y_t, y_t, \Delta'_t)_{t=1}^T$.
    From now on, we consider $\Delta_t'$ as the feedback in the learning process of a conservative algorithm since it make no difference from running the same algorithm with feedback $\Delta_t$.

    \paragraph{Smooth the data distribution} 
    For technical reasons (appearing later in the analysis), we don't want to analyze distribution $\{\cD_i|i\in [n]\}$ directly as the probability of $\Delta_t = \be_{i}$ is $0$ when $f_t(\be_{i}) = +1$ under distribution $\cD_i$.
    Instead, we consider the mixture of $\cD_i$ and another distribution $\cD_i''$, which is identical to $\cD_i$ except that $r(x) = d(x,\be_i)$ when $y=-$.
    More specifically, let $\cD_i' = (1-p) \cD_i + p  \cD_i''$ with some extremely small $p$, where 
    $\cD_i''$'s marginal distribution over $\cX\times \cY$ is still $\Unif(X_0)\times \Rad(1-6\epsilon)$; the radius is $r = 2$
    when $y=+$, ; and the radius is $r= d(x,\be_i)$ when $y=-$.
    For any data distribution $\cD$, let $\bP_\cD$ be the dynamics of $(f_1, y_1,\hat y_1,\Delta'_1,\ldots, f_T, y_T,\hat y_T,\Delta'_T)$ under $\cD$.
    According to Lemma~\ref{lmm:smooth}, by setting $p=\frac{\epsilon}{16n^2}$, when $T\leq \frac{n}{\epsilon}$, with high probability we never sample from $\cD_i''$ and have that for any $i,j\in [n]$
    \begin{equation}
        \abs{\bP_{\cD_{i}}(i_\out = j)-\bP_{\cD_{i}'}(i_\out = j)}\leq \frac{1}{8}\,.\label{eq:smooth-res-csv}
    \end{equation}

    From now on, we only consider distribution $\cD_i'$ instead of $\cD_i$. The readers might have the question that why not using $\cD_i'$ for construction directly. This is because $\cD_i'$ does not satisfy realizability and no hypothesis has zero loss under $\cD_i'$.

\paragraph{Information gain from different choices of $f_t$} In each round of interaction, the learner picks a predictor $f_t$, which can be out of $\cH$.
    Here we enumerate all choices of $f_t$.
    \begin{itemize}
        \item $f_t(\cdot) = 2\ind{\emptyset}-1$ predicts all points in $\cX$ by negative. 
        No matter what $i^*$ is, we will observe $(\Delta_t = x_t, y_t)\sim \Unif(X_0)\times \Rad(1-6\epsilon)$ and $\hat y_t = -$. 
        They are identically distributed for all $i^*\in [n]$, and thus, $\Delta_t'$ is also identically distributed. We cannot tell any information of $i^*$ from this round.
        \item $f_t=2\ind{a_t}-1$ for some $a_t\subset \cX$ s.t. $a\cap X_0\neq \emptyset$. 
        Then $\Delta_t = \Delta(x_t,f_t,r_t) = \Delta(x_t,f_t,2\alpha)$ since $r_t> 2\alpha$ and $d(x_t,f_t)\leq 2\alpha$, $\hat y_t = +$, $y_t\sim \Rad(1-6\epsilon)$. None of these depends on $i^*$ and again, the distribution of $(\hat y_t, y_t, \Delta'_t)$ is identical for all $i^*$ and we cannot tell any information of $i^*$ from this round.
        \item $f_t = 2\ind{a_t}-1$ for some non-empty $a_t\subset \{\be_1,\ldots,\be_n\}$. For rounds with $y_t = +$,  we have $\hat y_t = +$ and $\Delta_t = \Delta(x_t,f_t,2)$, which still not depend on $i^*$.
        Thus we cannot learn any information about $i^*$.
        But we can learn when $y_t=-$.
        For rounds with $y_t=-$, if $\Delta_t \in a_t$, then we could observe  $\hat y_t = +$
        and $\Delta_t' = \Delta_t$, which at least tells that $2\ind{\{\Delta_t\}}-1$ is not the target function (with high probability); if $\Delta_t \notin a_t$, then $\hat y_t = -$ and we observe $\Delta_t' = \perp$.
    \end{itemize}
    Therefore, we only need to focus on the rounds with $f_t = 2\ind{a_t}-1$ for some non-empty $a_t\subset \{\be_1,\ldots,\be_n\}$ and $y_t =-$.
    It is worth noting that drawing an example $x$ from $X_0$ uniformly, it is equivalent to uniformly drawing a permutation of $\cH$ such that the distances between $x$ and $h$ over all $h\in \cH$ are permuted according to it. Then $\Delta_t = \be_j$ iff $\be_j\in a_t$, $d(x, \be_j)\leq d(x,\be_{i^*})$ and $d(x,\be_j)\leq d(x,\be_l)$ for all $\be_l\in a_t$.  
    Let $k_t = \abs{a_t}$ denote the cardinality of $a_t$.
    In such rounds, under distribution $\cD_{i^*}'$, the distribution of $\Delta_t'$ are described as follows.
    \begin{enumerate}
        \item The case of $\be_{i^*} \in a_t$: For all $j\in a_t\setminus\{i^*\}$, with probability $\frac{1}{k_t}$, $d(x_t,\be_j) = \min_{\be_l\in a_t} d(x_t,\be_l)$ and thus, $\Delta_t' =\Delta_t = \be_j$ and $\hat y_t = +$ (mistake round). With probability $\frac{1}{k_t}$, we have $d(x_t,\be_{i^*}) = \min_{\be_l\in a_t} d(x_t,\be_l)$. If the example is drawn from $\cD_{i^*}$, we have $\Delta_t = x_t$ and $y_t = -$ (correct round), thus $\Delta_t' = \perp$. If the example is drawn from $\cD_{i^*}''$, we have we have $\Delta_t' = \Delta_t = \be_{i^*}$ and $y_t = +$ (mistake round). Therefore, according to the definition of $\Delta_t'$ (Eq~\eqref{eq:deltaprime}), we have
        \begin{align*}
            \Delta'_t = \begin{cases}
                \be_j & \text{w.p. } \frac{1}{k_t} \text{ for } \be_j \in a_t, j\neq i^*\\
                \be_{i^*} & \text{w.p. } \frac{1}{k_t}p \\
                \perp & \text{w.p. } \frac{1}{k_t}(1-p)\,.
            \end{cases}
        \end{align*}
        We denote this distribution by $P_\in(a_t,i^*)$.
        \item The case of $\be_{i^*} \notin a_t$: For all $j\in a_t$, with probability $\frac{1}{k_t+1}$,  then $d(x_t,\be_j) = \min_{\be_l\in a_t\cup \{\be_{i^*}\}} d(x_t,\be_l)$ and 
        thus, $\Delta_t = \be_j$ and $\hat y_t = +$ (mistake round). With probability $\frac{1}{k_t+1}$, we have $d(x,\be_{i^*}) < \min_{\be_l\in a_t} d(x_t,\be_l)$ and thus, $\Delta_t = x_t$, $\hat y_t = -$ (correct round), and $\Delta_t' =\perp$. Therefore, the distribution of $\Delta_t'$ is
        \begin{align*}
            \Delta'_t = \begin{cases}
                \be_j & \text{w.p. } \frac{1}{k_t+1} \text{ for } \be_j \in a_t\\
                \perp & \text{w.p. } \frac{1}{k_t+1}\,.
            \end{cases}
        \end{align*}
        We denote this distribution by $P_{\notin}(a_t)$.
    \end{enumerate}

    To measure the information obtained from $\Delta_t'$, we will utilize the KL divergence of the distribution of $\Delta_t'$ under the data distribution $\cD_{i^*}$ from that under a benchmark distribution.
    Let $\bar \cD = \frac{1}{n}\sum_{i\in n}\cD_i'$ denote the average distribution. 
    The process of sampling from $\bar \cD$ is equivalent to sampling $i^*$ uniformly at random from $[n]$ first and drawing a sample from $\cD_{i^*}$.
    Then under $\bar \cD$, for any $\be_j \in a_t$, we have
    \begin{align*}
        \Pr(\Delta'_t = \be_j) =& \Pr(i^*=j)\Pr(\Delta'_t = \be_j|i^*=j) +\Pr(i^*\in a_t\setminus\{j\})\Pr(\Delta'_t = \be_j|i^*\in a_t\setminus\{j\})\\
        &+ \Pr(i^*\notin a_t)\Pr(\Delta'_t = \be_j|i^*\notin a_t)\\=& \frac{1}{n} \cdot \frac{p}{k_t} +\frac{k_t-1}{n}\cdot \frac{1}{k_t} + \frac{n-k_t}{n}\cdot \frac{1}{k_t+1} = \frac{nk_t-1 +p(k_t+1)}{nk_t(k_t+1)}\,,
    \end{align*}
    and
    \begin{align*}
        \Pr(\Delta'_t = \perp) &= \Pr(i^*\in a_t)\Pr(\Delta'_t = \perp|i^*\in a_t)+ \Pr(i^*\notin a_t)\Pr(\Delta'_t = \perp|i^*\notin a_t)\\&=\frac{k_t}{n}\cdot \frac{1-p}{k_t} + \frac{n-k_t}{n}\cdot \frac{1}{k_t+1} = \frac{n+1 - p(k_t+1)}{n(k_t+1)}\,.
    \end{align*}
    Thus, the distribution of $\Delta'_t$ under $\bar\cD$ is
    \begin{align*}
            \Delta_t' = \begin{cases}
                \be_j & \text{w.p. } \frac{nk_t-1+p(k_t+1)}{nk_t(k_t+1)}\text{ for } \be_j \in a_t\\
                \perp & \text{w.p. } \frac{n+1-p(k_t+1)}{n(k_t+1)}\,.
            \end{cases}
    \end{align*}
    We denote this distribution by $\bar P(a_t)$.
    Next we will compute the KL divergences of $P_{\in}(a_t,i^*)$ and $P_{\notin}(a_t)$ from $\bar P(a_t)$. We will use the inequality $\log(1+x)\leq x$ for $x\geq 0$ in the following calculation. For any $i^*$ s.t. $\be_{i^*}\in a_t$, we have 
    \begin{align}
        &\KL{\bar P(a_t)}{P_{\in}(a_t,i^*)} \nonumber\\
        =& (k_t-1) \frac{nk_t-1+p(k_t+1)}{nk_t(k_t+1)}\log(\frac{nk_t-1+p(k_t+1)}{nk_t(k_t+1)}k_t) \nonumber\\
        &+\frac{nk_t-1+p(k_t+1)}{nk_t(k_t+1)}\log(\frac{nk_t-1+p(k_t+1)}{nk_t(k_t+1)}\cdot \frac{k_t}{p}) \nonumber\\
        &+ \frac{n+1-p(k_t+1)}{n(k_t+1)}\log(\frac{n+1-p(k_t+1)}{n(k_t+1)}\cdot \frac{k_t}{1-p}) \nonumber\\
        \leq & 0 + \frac{1}{k_t+1}\log(\frac{1}{p}) + \frac{2p}{k_t+1}
        =\frac{1}{k_t+1}\log(\frac{1}{p}) + \frac{2p}{k_t+1}\,,\label{eq:kl-in-csv}
        \end{align}
        and
        \begin{align}
            &\KL{\bar P(a_t)}{P_{\notin}(a_t)}\nonumber\\
            = & k_t \frac{nk_t-1+p(k_t+1)}{nk_t(k_t+1)}\log(\frac{nk_t-1+p(k_t+1)}{nk_t(k_t+1)}(k_t+1)) \nonumber\\
            &+ \frac{n+1-p(k_t+1)}{n(k_t+1)}\log (\frac{n+1-p(k_t+1)}{n(k_t+1)}(k_t+1))\nonumber \\
            \leq & 0+ \frac{n+1}{n^2(k_t+1)} =
            \frac{n+1}{n^2(k_t+1)}\label{eq:kl-out-csv}\,.
        \end{align}
\paragraph{Lower bound of the information} 
We utilize the information theoretical framework of proving lower bounds for linear bandits (Theorem 11 by \cite{rajaraman2023beyond}) here.
For notation simplicity, for all $i\in [n]$, let $\bP_i$ denote the dynamics of $(f_1,\Delta'_1, y_1,\hat y_1,\ldots, f_T,\Delta'_T, y_T,\hat y_T)$ under $\cD_i'$ and $\bar \bP$ denote the dynamics under $\bar \cD$.
Let $B_t$ denote the event of $\{f_t = 2\ind{a_t}-1 \text{ for some non-empty } a_t\subset \{\be_1,\ldots,\be_n\}\}$. 
As discussed before, for any $a_t$, conditional on $\neg B_t$ or $y_t =+1$,  $(\Delta_t', y_t,\hat y_t)$ are identical in all $\{\cD_i'|i\in [n]\}$, and therefore, also identical in $\bar \cD$.
We can only obtain information at rounds when $B_t \wedge (y_t =-1)$ occurs.
In such rounds, we know that $f_t$ is fully determined by history (possibly with external randomness , which does not depend on data distribution), $y_t =-1$ and $\hat y_t$ is fully determined by $\Delta_t'$ ($\hat y_t = +1$ iff. $\Delta_t' \in a_t$).

Therefore, conditional the history $H_{t-1} = (f_1,\Delta_1', y_1,\hat y_1,\ldots, f_{t-1},\Delta_{t-1}', y_{t-1},\hat y_{t-1})$ before time $t$, we have
\begin{align}
    &\KL{\bar \bP(f_{t},\Delta_{t}', y_{t},\hat y_{t}|H_{t-1})}{\bP_i( f_{t},\Delta_{t}', y_{t},\hat y_{t}|H_{t-1})} \nonumber\\
    =&\bar \bP(B_t  \wedge (y_t =-1)) \KL{\bar \bP(\Delta_t'|H_{t-1},B_t  \wedge (y_t =-1))}{\bP_i(\Delta_t'|H_{t-1},B_t  \wedge (y_t =-1))} \nonumber\\
    =& 6\epsilon \bar \bP(B_t) \KL{\bar \bP(\Delta_t'|H_{t-1},B_t  \wedge (y_t =-1))}{\bP_i(\Delta_t'|H_{t-1},B_t  \wedge (y_t =-1))}\label{eq:kl-delta-csv}\,,
\end{align}
where the last equality holds due to that $y_t\sim \Rad(1-6\epsilon)$ and does not depend on $B_t$.

For any algorithm that can successfully identify $i$ under the data distribution $\cD_i$ with probability $\frac{3}{4}$ for all $i\in [n]$, then $\bP_{\cD_i}(i_\out = i)\geq \frac{3}{4}$ and $\bP_{\cD_j}(i_\out = i)\leq \frac{1}{4}$ for all $j\neq i$.
Recall that $\cD_i$ and $\cD_i'$ are very close when the mixture parameter $p$ is small. Combining with Eq~\eqref{eq:smooth-res-csv}, we have 
\begin{align*}
    &\abs{\bP_{i}(i_\out = i) - \bP_{j}(i_\out = i)}\\
    \geq& \abs{\bP_{\cD_i}(i_\out = i) - \bP_{\cD_j}(i_\out = i)} - \abs{\bP_{\cD_i}(i_\out = i)- \bP_{i}(i_\out = i)}- \abs{\bP_{\cD_j}(i_\out = i)- \bP_{j}(i_\out = i)} \\
    \geq&\frac{1}{2}- \frac{1}{4} = \frac{1}{4}\,.
\end{align*}

Then we have the total variation distance between $\bP_{i}$ and $\bP_{j}$
\begin{align}
       \TV(\bP_{i},\bP_{j})\geq \abs{\bP_{i}(i_\out = i) - \bP_{j}(i_\out = i)}\geq \frac{1}{4}\,.\label{eq:tv-lb-csv}
\end{align}


Then we have
    \begin{align*}
    &\EEs{i\sim \Unif([n])}{\TV^2(\bP_{i},\bP_{(i+1) \text{ mod } n})} \leq 4 \EEs{i\sim \Unif([n])}{\TV^2(\bP_{i},\bar \bP )}\\
    \leq& 2\EEs{i}{\KL{\bar \bP}{\bP_{i}}} \tag{Pinsker's ineq}\\
    =& 2\EEs{i}{\sum_{t=1}^T \KL{\bar \bP(f_t,\Delta'_t, y_t,\hat y_t|H_{t-1})}{\bP_i(f_t,\Delta'_t, y_t,\hat y_t|H_{t-1})}}\tag{Chain rule}\\
    =& 12\epsilon \EEs{i}{\sum_{t=1}^T \bar \bP(B_t) \KL{\bar \bP(\Delta_t'|H_{t-1},B_t  \wedge (y_t =-1))}{\bP_i(\Delta_t'|H_{t-1},B_t  \wedge (y_t =-1))}}\tag{Apply Eq~\eqref{eq:kl-delta-csv}}\\
     =& \frac{12\epsilon}{n} \sum_{t=1}^T \bar \bP(B_t)\sum_{i=1}^n \KL{\bar \bP(\Delta_t'|H_{t-1},B_t  \wedge (y_t =-1))}{\bP_i(\Delta_t'|H_{t-1},B_t  \wedge (y_t =-1))}\\
     =& \frac{12\epsilon}{n}  \EEs{f_{1:T}\sim \bar \bP}{\sum_{t=1}^T\true{B_t}\left(\sum_{i:i\in a_t} \KL{\bar P(a_t)}{P_{\in}(a_t,i)}
     + \sum_{i:i\notin a_t}\KL{\bar P(a_t)}{P_{\notin}(a_t)}\right)}\\
     \leq & \frac{12\epsilon}{n} \sum_{t=1}^T \EEs{f_{1:T}\sim \bar \bP}{\sum_{i:i\in a_t} \left(\frac{1}{k_t+1}\log(\frac{1}{p}) + \frac{2p}{k_t+1}
     \right) + \sum_{i:i\notin a_t}\frac{n+1}{n^2(k_t+1)}}\tag{Apply Eq~\eqref{eq:kl-in-csv},\eqref{eq:kl-out-csv}}\\
     \leq & \frac{12\epsilon}{n} \sum_{t=1}^T (\log(\frac{1}{p})+ 2p + 1)\\
     \leq & \frac{12T\epsilon(\log(16n^2/\epsilon) +2)}{n}\,.
\end{align*}
Combining with Eq~\eqref{eq:tv-lb-csv}, we have that there exists a universal constant $c$ such that $T\geq \frac{cn}{\epsilon (\log(n/\epsilon) +1 )}$.
\end{proof}
\section{Proof of Theorem~\ref{thm:x-delta-never}}\label{app:x-delta-never}
\begin{proof} We will prove Theorem~\ref{thm:x-delta-never} by constructing an instance of $\cQ$ and $\cH$ and then reduce it to a linear stochastic bandit problem.
\paragraph{Construction of $\cQ, \cH$ and a set of realizable distributions}
\begin{itemize}
    \item Consider the input metric space in the shape of a star, where $\cX=\{0,1,\ldots,n\}$ and the distance function of $d(0,i) = 1$ and $d(i,j) =2$ for all $i\neq j\in [n]$.
    \item Let the hypothesis class be a set of singletons over $[n]$, i.e., $\cH =\{2\ind{\{i\}}-1|i\in [n]\}$.
    \item We define a collection of distributions $\{\cD_i|i\in [n]\}$ in which $\cD_i$ is realized by $2\ind{\{i\}}-1$. The data distribution $\cD_{i}$ put $1-3(n-1)\epsilon$ on $(0,1,+)$ and $3\epsilon$ on $(i,1,-)$ for all $i\neq i^*$. Hence, note that all distributions 
    in $\{\cD_i|i\in [n]\}$ share the same distribution support $\{(0,1,+)\}\cup\{(i,1,-)|i\in [n]\}$, but have different weights.
\end{itemize}

\paragraph{Randomization and improperness of the output $f_\out$ do not help.}
Note that algorithms are allowed to output a randomized $f_\out$ and to output $f_\out\notin \cH$.
We will show that randomization and improperness of $f_\out$ don't make the problem easier.
Supposing that the data distribution is $\cD_{i^*}$ for some $i^*\in [n]$, finding a (possibly randomized and improper) $f_\out$ is not easier than identifying $i^*$.
Since our feature space $\cX$ is finite, we can enumerate all hypotheses not equal to $2\ind{\{i^*\}}-1$ and calculate their strategic population loss as follows.
The hypothesis $2\ind{\emptyset}-1$ will predict all by negative and thus $\err(2\ind{\emptyset}-1) = 1-3(n-1)\epsilon$.
For any hypothesis predicting $0$ by positive, it will predict all points in the distribution support by positive and thus incurs strategic loss $3(n-1)\epsilon$.
For any hypothesis predicting $0$ by negative and some $i\neq i^*$ by positive, then it will misclassify $(i,1,-)$ and incur strategic loss $3\epsilon$.
Therefore, for any hypothesis $h\neq 2\ind{\{i^*\}}-1$, we have $\err_{\cD_{i^*}}(h)\geq 3\epsilon$. 

Similar to the proof of Theorem~\ref{thm:delta-csv}, under distribution $\cD_{i^*}$, if we are able to find a (possibly randomized) $f_\out$ with strategic loss $\err(f_\out)\leq \epsilon$.
Then $\Pr_{h\sim f_\out}(h=2\ind{\{i^*\}}-1) \geq \frac{2}{3}$.
We can identify $i^*$ by checking which realization of $f_\out$ has probability greater than $\frac{2}{3}$.
In the following, we will focus on the sample complexity to identify the target function $2\ind{\{i^*\}}-1$ or simply $i^*$.
Let $i_\out$ denote the algorithm's answer to question of ``what is $i^*$?''.

\textbf{Smooth the data distribution}
For technical reasons (appearing later in the analysis), we don't want to analyze distribution $\{\cD_i|i\in [n]\}$ directly as the probability of $(i,1,-)$ is $0$ under distribution $\cD_i$.
Instead, for each $i\in [n]$, let $\cD_i' = (1-p)\cD_i + p \cD_i''$ be the mixture of $\cD_i$ and $\cD_i''$ for some small $p$, where $\cD_i'' = (1-3(n-1)\epsilon)\ind{\{(0,1,+)\}} + 3(n-1)\epsilon \ind{\{(i,1,-)\}}$.
%
Specifically,
\begin{align*}
    \cD_i'(z) = \begin{cases}
        1-3(n-1)\epsilon & \text{for }z = (0,1,+)\\
        3(1-p)\epsilon & \text{for } z = (j,1,-),\forall j\neq i\\
        3(n-1)p\epsilon & \text{for } z = (i,1,-)
    \end{cases}
\end{align*}
For any data distribution $\cD$, let $\bP_\cD$ be the dynamics of $(f_1, y_1,\hat y_1,\ldots, f_T, y_T,\hat y_T)$ under $\cD$.
    According to Lemma~\ref{lmm:smooth}, by setting $p=\frac{\epsilon}{16n^2}$, when $T\leq \frac{n}{\epsilon}$, we have that for any $i,j\in [n]$
    \begin{equation}
        \abs{\bP_{\cD_{i}}(i_\out = j)-\bP_{\cD_{i}'}(i_\out = j)}\leq \frac{1}{8}\,.\label{eq:smooth-res-never}
    \end{equation}
    From now on, we only consider distribution $\cD_i'$ instead of $\cD_i$.
The readers might have the question that why not using $\cD_i'$ for construction directly. This is because no hypothesis has zero loss under $\cD_i'$, and thus $\cD_i'$ does not satisfy realizability requirement.

\paragraph{Information gain from different choices of $f_t$}
Note that in each round, the learner picks a $f_t$ and then only observes $\hat y_t$ and $y_t$. 
Here we enumerate choices of $f_t$ as follows.
\begin{enumerate}
    \item $f_t = 2\ind{\emptyset}-1$ predicts all points in $\cX$ by negative. No matter what $i^*$ is, we observe $\hat y_t = -$ and $y_t = 2\1{x_t = 0}-1$.
    Hence $(\hat y_t,y_t)$ are identically distributed for all $i^*\in [n]$, and thus, we cannot learn anything about $i^*$ from this round.
    \item $f_t$ predicts $0$ by positive. Then no matter what $i^*$ is, we have $\hat y_t = +$ and $y_t = \1{x_t = 0}$. Thus again, we cannot learn anything about $i^*$.
    \item $f_t = 2\ind{a_t}-1$ for some non-empty $a_t\subset [n]$. For rounds with $x_t =0$, we have $\hat y_t = y_t = +$ no matter what $i^*$ is and thus, we cannot learn anything about $i^*$. For rounds with $y_t=-$, i.e., $x_t\neq 0$, we will observe $\hat y_t = f_t(\Delta(x_t,f_t, 1)) = \1{x_t\in a_t}$. 
\end{enumerate}
Hence, we can only extract information with the third type of $f_t$ at rounds with $x_t\neq 0$.

\paragraph{Reduction to stochastic linear bandits}
In rounds with  $f_t = 2\ind{a_t}-1$ for some non-empty $a_t\subset [n]$ and $x_t\neq 0$, our problem is identical to a stochastic linear bandit problem.
Let us state our problem as Problem~\ref{prob:ours} and a linear bandit problem as Problem~\ref{prob:bandit}. Let $A = \{0,1\}^n\setminus \{\bZero\}$.
\begin{problem}\label{prob:ours}
The environment picks an $i^*\in [n]$. At each round $t$, the environment picks $x_t \in \{\be_i|i\in [n]\}$ with $P(i) = \frac{1-p}{n-1}$ for $i\neq i^*$ and $P(i^*) = p$ and the learner picks an $a_t\in A$ (where we use a $n$-bit string to represent $a_t$ and $a_{t,i} = 1$ means that $a_t$ predicts $i$ by positive). Then the learner observes $\hat y_t = \1{a_t^\top x_t >0}$ (where we use $0$ to represent nagative label).
\end{problem}

\begin{problem}\label{prob:bandit}
The environment picks a linear parameter $w^* \in \{w^i|i\in [n]\}$ with $w^i = \frac{1-p}{n-1}\bOne-(\frac{1-p}{n-1} -p)\be_i$.
The arm set is $A$. For each arm $a\in A$, the reward is i.i.d. from the following distribution:
\begin{align}
    r_w(a) = \begin{cases}
    -1, \text{ w.p. } w^\top a\,,\\
    0\,.
    \end{cases}\label{eq:dist-reward}
\end{align}
If the linear parameter $w^* = w^{i^*}$, the optimal arm is $\be_{i^*}$.

\end{problem}

\begin{claim}
For any $\delta>0$, for any algorithm $\cA$ that identify $i^*$ correctly with probability $1-\delta$ within $T$ rounds for any $i^*\in [n]$ in Problem~\ref{prob:ours}, we can construct another algorithm $\cA'$ can also identify the optimal arm in any environment with probability $1-\delta$ within $T$ rounds in Problem~\ref{prob:bandit}.
\end{claim}
This claim follows directly from the problem descriptions. Given any algorithm $\cA$ for Problem~\ref{prob:ours}, we can construct another algorithm $\cA'$ which simulates $\cA$.
At round $t$, if $\cA$ selects predictor $a_t$, then $\cA'$ picks arm the same as $a_t$.
Then $\cA'$ observes a reward $r_{w^{i^*}}(a_t)$, which is $-1$ w.p. $w^{i^*\top} a_t$ and feed $-r_{w^{i^*}}(a_t)$ to $\cA$. Since $\hat y_t$ in Problem~\ref{prob:ours} is $1$ w.p. $\sum_{i=1}^n a_{t,i}P(i) = w^{i^*\top}a_t$, it is distributed identically as $-r_{w^{i^*}}(a_t)$. Since $\cA$ will be able to identify $i^*$ w.p. $1-\delta$ in $T$ rounds, $\cA'$ just need to output $\be_{i^*}$ as the optimal arm.

Then any lower bound on $T$ for Problem \ref{prob:bandit} also lower bounds Problem \ref{prob:ours}. Hence, we adopt the information theoretical framework of proving lower bounds for linear bandits (Theorem 11 by \cite{rajaraman2023beyond}) to prove a lower bound for our problem. In fact, we also apply this framework to prove the lower bounds in other settings of this work, including Theorem~\ref{thm:delta-csv} and Theorem~\ref{thm:non-ball}.

\paragraph{Lower bound of the information}

For notation simplicity, for all $i\in [n]$, let $\bP_i$ denote the dynamics of $(f_1, y_1,\hat y_1,\ldots, f_T, y_T,\hat y_T)$ under $\cD_i'$ and and $\bar \bP$ denote the dynamics under $\bar \cD = \frac{1}{n}\cD_{i}'$.
Let $B_t$ denote the event of $\{f_t = 2\ind{a_t}-1 \text{ for some non-empty } a_t\subset [n]\}$.
As discussed before, for any $a_t$, conditional on $\neg B_t$ or $y_t =+1$,  $(x_t, y_t,\hat y_t)$ are identical in all $\{\cD_i'|i\in [n]\}$, and therefore, also identical in $\bar \cD$.
We can only obtain information at rounds when $B_t \wedge y_t =-1$ occurs.
In such rounds, $f_t$ is fully determined by history (possibly with external randomness , which does not depend on data distribution), $y_t =-1$ and $\hat y_t = -r_w(a_t)$ with $r_w(a_t)$ sampled from the distribution defined in Eq~\eqref{eq:dist-reward}.

For any algorithm that can successfully identify $i$ under the data distribution $\cD_i$ with probability $\frac{3}{4}$ for all $i\in [n]$, then $\bP_{\cD_i}(i_\out = i)\geq \frac{3}{4}$ and $\bP_{\cD_j}(i_\out = i)\leq \frac{1}{4}$ for all $j\neq i$.
Recall that $\cD_i$ and $\cD_i'$ are very close when the mixture parameter $p$ is small. Combining with Eq~\eqref{eq:smooth-res-never}, we have 
\begin{align}
    &\abs{\bP_{i}(i_\out = i) - \bP_{j}(i_\out = i)}\nonumber\\
    \geq& \abs{\bP_{\cD_i}(i_\out = i) - \bP_{\cD_j}(i_\out = i)} - \abs{\bP_{\cD_i}(i_\out = i)- \bP_{i}(i_\out = i)}- \abs{\bP_{\cD_j}(i_\out = i)- \bP_{j}(i_\out = i)} \nonumber\\
    \geq&\frac{1}{2}- \frac{1}{4} = \frac{1}{4}\,.\label{eq:tv-lb-never}
\end{align}


Let $\bar w = \frac{1}{n}\bOne$.
Let $\kl{q}{q'}$ denote the KL divergence from $\Ber(q)$ to $\Ber(q')$. Let $H_{t-1} = (f_1,y_1,\hat y_1,\ldots,f_{t-1},y_{t-1},\hat y_{t-1})$ denote the history up to time $t-1$.
Then we have
\begin{align}
    &\EEs{i\sim \Unif([n])}{\TV^2(\bP_i,\bP_{{i+1} \text{ mod } n})} \leq 4 \EEs{i\sim \Unif([n])}{\TV^2(\bP_i,\bar \bP)}\nonumber\\
    \leq& 2\EEs{i}{\KL{\bar \bP}{\bP_i}} \tag{Pinsker's ineq}\nonumber\\
    =& 2\EEs{i}{\sum_{t=1}^T \KL{\bar \bP(f_t, y_t,\hat y_t|H_{t-1})}{\bP_i(f_t, y_t,\hat y_t|H_{t-1})}}\tag{Chain rule}\nonumber\\
    =& 2\EEs{i}{\sum_{t=1}^T \bar \bP(B_t\wedge y_t=-1)\EEs{a_{1:T}\sim \bar \bP}{ \KL{\Ber(\inner{\bar w}{a_t})}{\Ber(\inner{w^i}{a_t})}}}\nonumber\\
    =& 6(n-1)\epsilon\EEs{i}{\sum_{t=1}^T \bar \bP(B_t)\EEs{a_{1:T}\sim \bar \bP}{ \KL{\Ber(\inner{\bar w}{a_t})}{\Ber(\inner{w^i}{a_t})}}}\nonumber\\
     =& \frac{6(n-1)\epsilon}{n} \sum_{t=1}^T \EEs{a_{1:T}\sim \bar \bP}{\sum_{i=1}^n \KL{\Ber(\inner{\bar w}{a_t})}{\Ber(\inner{w^i}{a_t})}}\nonumber\\
     =& \frac{6(n-1)\epsilon}{n} \sum_{t=1}^T \EEs{a_{1:T}\sim \bar \bP}{\sum_{i:i\in a_t}\kl{\frac{k_t}{n}}{ \frac{(k_t-1)(1-p)}{n-1}+ p} 
     + \sum_{i:i\notin a_t}\kl{\frac{k_t}{n}}{ \frac{k_t(1-p)}{n-1}}}\nonumber\\
     =& \frac{6(n-1)\epsilon}{n} \sum_{t=1}^T \EEs{a_{1:T}\sim \bar \bP}{k_t \kl{\frac{k_t}{n}}{ \frac{(k_t-1)(1-p)}{n-1}+ p} 
     + (n-k_t)\kl{\frac{k_t}{n}}{ \frac{k_t(1-p)}{n-1}}}\label{eq:info-ub-never}
\end{align}
If $k_t = 1$, then
\begin{align*}
    k_t\cdot \kl{\frac{k_t}{n}}{ \frac{(k_t-1)(1-p)}{n-1}+ p} = \kl{\frac{1}{n}}{ p}\leq \frac{1}{n} \log(\frac{1}{p})\,,
\end{align*}
and 
\begin{align*}
    (n-k_t)\cdot\kl{\frac{k_t}{n}}{ \frac{k_t(1-p)}{n-1}}= (n-1)\cdot\kl{\frac{1}{n}}{\frac{1-p}{n-1}}\leq \frac{1}{(1-p)n(n-2)}\,,
\end{align*}
where the ineq holds due to $\kl{q}{q'}\leq \frac{(q-q')^2}{q'(1-q')}$.
If $k_t = n-1$, it is symmetric to the case of $k_t = 1$. We have
\begin{align*}
    &k_t\cdot \kl{\frac{k_t}{n}}{ \frac{(k_t-1)(1-p)}{n-1}+ p} = (n-1)\kl{\frac{n-1}{n}}{\frac{n-2}{n-1} +\frac{1}{n-1}p}=
    (n-1)\kl{\frac{1}{n}}{\frac{1-p}{n-1}}\\
    \leq& \frac{1}{(1-p)n(n-2)}\,,
\end{align*}
and 
\begin{align*}
    (n-k_t)\cdot\kl{\frac{k_t}{n}}{ \frac{k_t(1-p)}{n-1}}= \kl{\frac{n-1}{n}}{1-p}= 
     \kl{\frac{1}{n}}{ p}\leq \frac{1}{n} \log(\frac{1}{p})\,.   
\end{align*}
If $1<k_t<n-1$, then
\begin{align*}
    k_t\cdot\kl{\frac{k_t}{n}}{ \frac{(k_t-1)(1-p)}{n-1}+ p} = &k_t\cdot\kl{\frac{k_t}{n}}{ \frac{k_t-1}{n-1} +\frac{n-k_t}{n-1}p}\stackrel{(a)}{\leq} 
    k_t\cdot\kl{\frac{k_t}{n}}{ \frac{k_t-1}{n-1}}\\
    \stackrel{(b)}{\leq}&
    k_t\cdot\frac{(\frac{k_t}{n}-\frac{k_t-1}{n-1} )^2}{\frac{k_t-1}{n-1}(1-\frac{k_t-1}{n-1})}= k_t\cdot\frac{n-k_t}{n^2(k_t-1)}\leq \frac{k_t\cdot}{n(k_t-1)}\leq \frac{2}{n}\,,
\end{align*}
where inequality (a) holds due to that $\frac{k_t-1}{n-1} +\frac{n-k_t}{n-1}p \leq \frac{k_t}{n}$ and $\kl{q}{q'}$ is monotonically decreasing in $q'$ when $q'\leq q$ and inequality (b) adopts $\kl{q}{q'}\leq \frac{(q-q')^2}{q'(1-q')}$,
and
\begin{align*}
    (n-k_t)\cdot\kl{\frac{k_t}{n}}{ \frac{k_t(1-p)}{n-1}}\leq (n-k_t)\cdot\kl{\frac{k_t}{n}}{ \frac{k_t}{n-1}}\leq \frac{k_t(n-k_t)}{n^2(n-1-k_t)}\leq \frac{2k_t}{n^2}\,,
\end{align*}
where the first inequality hold due to that $\frac{k_t(1-p)}{n-1} \geq \frac{k_t}{n}$, and $\kl{q}{q'}$ is monotonically increasing in $q'$ when $q'\geq q$ and the second inequality adopts $\kl{q}{q'}\leq \frac{(q-q')^2}{q'(1-q')}$.
Therefore, we have
\begin{equation*}
    \text{Eq~\eqref{eq:info-ub-never}} \leq \frac{6(n-1)\epsilon}{n} \sum_{t=1}^T \EEs{a_{1:T}\sim \bar \bP}{\frac{2}{n}\log(\frac{1}{p})} \leq \frac{12\epsilon T\log(1/p)}{n} \,.
\end{equation*}
Combining with Eq~\eqref{eq:tv-lb-never}, we have that there exists a universal constant $c$ such that $T\geq \frac{cn}{\epsilon (\log(n/\epsilon) +1 )}$.
\end{proof}
\section{Proof of Theorem~\ref{thm:non-ball}}\label{app:non-ball}

\begin{proof}
We will prove Theorem~\ref{thm:non-ball} by constructing an instance of $\cQ$ and $\cH$ and showing that for any learning algorithm, there exists a realizable data distribution s.t. achieving $\epsilon$ loss requires at least $\tilde \Omega(\frac{\abs{\cH}}{\epsilon})$ samples.
\paragraph{Construction of $\cQ$, $\cH$ and a set of realizable distributions} 
    \begin{itemize}
        \item Let feature vector space $\cX = \{0,1,\ldots,n\}$ and let the space of feature-manipulation set pairs $\cQ = \{(0,\{0\}\cup s)|s\subset [n]\}$.
        That is to say, every agent has the same original feature vector $x =0$ but has different manipulation ability according to $s$.
        \item Let the hypothesis class be a set of singletons over $[n]$, i.e., $\cH = \{2\ind{\{i\}}-1|i\in [n]\}$.
        \item We now define a collection of distributions $\{\cD_i|i\in [n]\}$ in which $\cD_i$ is realized by $2\ind{\{i\}}-1$.
        For any $i\in [n]$, let $\cD_{i}$ put probability mass $1-6\epsilon$ on $(0,\cX,+1)$ and $6\epsilon$ uniformly over $\{(0, \{0\}\cup s_{\sigma, i}, -1)|\sigma \in \cS_n\}$, where $\cS_n$ is the set of all permutations over $n$ elements and $s_{\sigma, i}:= \{j|\sigma^{-1}(j)< \sigma^{-1}(i)\}$ is the set of elements appearing before $i$ in the permutation $(\sigma(1),\ldots,\sigma(n))$.
        In other words, with probability $1-6\epsilon$, we will sample $(0,\cX,+1)$ and with $\epsilon$, we will randomly draw a permutation $\sigma\sim \Unif(\cS_n)$ and return $(0, \{0\}\cup s_{\sigma, i}, -1)$. 
        The data distribution $\cD_{i}$ is realized by $2\ind{\{i\}}-1$ since for negative examples $(0, \{0\}\cup s_{\sigma, i}, -1)$, we have $i\notin s$ and for positive examples $(0,\cX,+1)$, we have $i\in \cX$.
    \end{itemize}

    \paragraph{Randomization and improperness of the output $f_\out$ do not help}
    Note that algorithms are allowed to output a randomized $f_\out$ and to output $f_\out\notin \cH$.
    We will show that randomization and improperness of $f_\out$ don't make the problem easier.
    That is, supposing that the data distribution is $\cD_{i^*}$ for some $i^*\in [n]$, finding a (possibly randomized and improper) $f_\out$ is not easier than identifying $i^*$.
    Since our feature space $\cX$ is finite, we can enumerate all hypotheses not equal to $2\ind{\{i^*\}}-1$ and calculate their strategic population loss as follows.
    \begin{itemize}
        \item $2\ind{\emptyset}-1$ predicts all points in $\cX$ by negative and thus
    $\err(2\ind{\emptyset}-1) = 1-6\epsilon$;
    \item For any $a\subset \cX$ s.t. $0\in a$, $2\ind{a}-1$ will predict $0$ as positive and thus will predict any point drawn from $\cD_{i^*}$ as positive.
    Hence
    $\err(2\ind{a}-1) = 6\epsilon$;
    \item For any $a\subset [n]$ s.t. $\exists i\neq i^*$, $i\in a$, we have  $\err(2\ind{a}-1)\geq 3\epsilon$. This is due to that when $y=-1$, the probability of drawing a permutation $\sigma$ with $\sigma^{-1}(i)<\sigma^{-1}(i^*)$ is $\frac{1}{2}$. In this case, we have $i\in s_{\sigma,i^*}$ and the prediction of $2\ind{a}-1$ is $+1$.
    \end{itemize}
    Under distribution $\cD_{i^*}$, if we are able to find a (possibly randomized) $f_\out$ with strategic loss $\err(f_\out)\leq \epsilon$, then we have $\err(f_\out) = \EEs{h\sim f_\out}{\err(h)}\geq \Pr_{h\sim f_\out}(h\neq 2\ind{\{i^*\}}-1) \cdot 3\epsilon$.
    Thus, $\Pr_{h\sim f_\out}(h= 2\ind{\{i^*\}}-1)\geq \frac{2}{3}$ and then, we can identify $i^*$ by checking which realization of $f_\out$ has probability greater than $\frac{2}{3}$.
    In the following, we will focus on the sample complexity to identify the target function $2\ind{\{i^*\}}-1$ or simply $i^*$.
    Let $i_\out$ denote the algorithm's answer to question of ``what is $i^*$?''.

\paragraph{Smoothing the data distribution} 
    For technical reasons (appearing later in the analysis), we don't want to analyze distribution $\{\cD_i|i\in [n]\}$ directly as the probability of $\Delta_t = i^*$ is $0$ when $f_t(i^*) = +1$.
    Instead, we consider the mixture of $\cD_i$ and another distribution $\cD_i''$ to make the probability of $\Delta_t = i^*$ be a small positive number.
    More specifically, let $\cD_i' = (1-p) \cD_i + p  \cD_i''$, where $\cD_i''$
    is defined by drawing $(0,\cX, +1)$ with probability $1-6\epsilon$ and $(0,\{0,i\}, -1)$ with probability $6\epsilon$. When $p$ is extremely small, we will never sample from $\cD_i''$ when time horizon $T$ is not too large and therefore, the algorithm behaves the same under $\cD_i'$ and $\cD_i$.
    For any data distribution $\cD$, let $\bP_\cD$ be the dynamics of $(x_1,f_1,\Delta_1, y_1,\hat y_1,\ldots, x_T, f_T,\Delta_T, y_T,\hat y_T)$ under $\cD$.
    According to Lemma~\ref{lmm:smooth}, by setting $p=\frac{\epsilon}{16n^2}$, when $T\leq \frac{n}{\epsilon}$, we have that for any $i,j\in [n]$
    \begin{equation}
        \abs{\bP_{\cD_{i}}(i_\out = j)-\bP_{\cD_{i}'}(i_\out = j)}\leq \frac{1}{8}\,.\label{eq:smooth-res}
    \end{equation}
   
    From now on, we only consider distribution $\cD_i'$ instead of $\cD_i$. The readers might have the question that why not using $\cD_i'$ for construction directly. This is because no hypothesis has zero loss under $\cD_i'$, and thus $\cD_i'$ does not satisfy realizability requirement.

    \paragraph{Information gain from different choices of $f_t$} In each round of interaction, the learner picks a predictor $f_t$, which can be out of $\cH$.
    Suppose that the target function is $2\ind{\{i^*\}}-1$ .
    Here we enumerate all choices of $f_t$ and discuss how much we can learn from each choice.
    \begin{itemize}
        \item $f_t = 2\ind{\emptyset}-1$ predicts all points in $\cX$ by negative. 
        No matter what $i^*$ is, we will observe $\Delta_t = x_t = 0$, $y_t\sim \Rad(1-6\epsilon)$, $\hat y_t = -1$. They are identically distributed for any $i^*\in [n]$ and thus we cannot tell any information of $i^*$ from this round.

        \item $f_t=2\ind{a_t}-1$ for some $a_t\subset \cX$ s.t. $0\in a_t$. 
        Then no matter what $i^*$ is, we will observe $\Delta_t = x_t = 0$, $y_t\sim \Rad(1-6\epsilon)$, $\hat y_t = +1$.  Again, we cannot tell any information of $i^*$ from this round.

        \item $f_t = 2\ind{a_t}-1$ for some some non-empty $a_t\subset [n]$. 
        For rounds with $y_t = +1$,  we have $x_t =0, \hat y_t = +1$ and $\Delta_t = \Delta(0, f_t, \cX)\sim \Unif(a_t)$,
        which still do not depend on $i^*$.
        For rounds with $y_t=-1$, if the drawn example $(0,\{0\}\cup s, -1)$ satisfies that $s\cap a_t\neq \emptyset$, the we would observe $\Delta_t \in a_t$ and $\hat y_t = +1$.
        At least we could tell that $\ind{\{\Delta_t\}}$ is not the target function.
       Otherwise, we would observe $\Delta_t=x_t=0$ and $\hat y_t = -1$.
    \end{itemize}
    Therefore, we can only gain some information about $i^*$ at rounds in which $f_t = 2\ind{a_t}-1$ for some non-empty $a_t\subset [n]$ and $y_t =-1$.
    In such rounds, under distribution $\cD_{i^*}'$, the distribution of $\Delta_t$ is described as follows.
    Let $k_t = \abs{a_t}$ denote the cardinality of $a_t$.
    Recall that agent $(0,\{0\}\cup s, -1)$ breaks ties randomly when choosing $\Delta_t$ if there are multiple elements in $a_t\cap s$.
    Here are two cases: $i^* \in a_t$ and $i^* \notin a_t$.
    \begin{enumerate}
        \item The case of $i^* \in a_t$: With probability $p$, we are sampling from $\cD_{i^*}''$ and then $\Delta_t = i^*$.
        With probability $1-p$, we are sampling from $\cD_{i^*}$. 
        Conditional on this, with probability $\frac{1}{k_t}$, we sample an agent $(0,\{0\}\cup s_{\sigma, i^*}, -1)$ with the permutation $\sigma$ satisfying that $\sigma^{-1}(i^*)<\sigma^{-1}(j)$ for all $j\in a_t\setminus \{i^*\}$ and thus, $\Delta_t = 0$.
        With probability $1-\frac{1}{k_t}$, there exists $j\in a_t\setminus \{i^*\}$ s.t. $\sigma^{-1}(j)<\sigma^{-1}(i^*)$ and $\Delta_t \neq 0$. Since all $j\in a_t\setminus \{i^*\}$ are symmetric, we have $\Pr(\Delta_t = j) = (1-p)(1-\frac{1}{k_t}) \cdot \frac{1}{k_t-1} = \frac{1-p}{k_t}$.
        Hence, the distribution of $\Delta_t$ is
        \begin{align*}
            \Delta_t = \begin{cases}
                j & \text{w.p. } \frac{1-p}{k_t} \text{ for } j \in a_t, j\neq i^*\\
                i^* & \text{w.p. } p \\
               0 & \text{w.p. } \frac{1-p}{k_t}\,.
            \end{cases}
        \end{align*}
        We denote this distribution by $P_\in(a_t, i^*)$.
        \item The case of $i^*\notin a_t$: With probability $p$, we are sampling from $\cD_{i^*}''$, we have $\Delta_t=x_t=0$.
        With probability $1-p$, we are sampling from $\cD_{i^*}$.
        Conditional on this, with probability of $\frac{1}{k_t+1}$, $\sigma^{-1}(i^*)<\sigma^{-1}(j)$ for all $j\in a_t$ and thus, $\Delta_t =x_t= 0$.
        With probability $1-\frac{1}{k_t+1}$
        there exists $j\in a_t$ s.t. $\sigma^{-1}(j)<\sigma^{-1}(i^*)$ and $\Delta_t \in a_t$. Since all $j\in a_t$ are symmetric, we have $\Pr(\Delta_t = j) = (1-p)(1-\frac{1}{k_t+1}) \cdot \frac{1}{k_t} = \frac{1-p}{k_t+1}$.
        Hence the distribution of $\Delta_t$ is
        \begin{align*}
            \Delta_t = \begin{cases}
                j & \text{w.p. } \frac{1-p}{k_t+1} \text{ for } j \in a_t\\
                0 & \text{w.p. } p+ \frac{1-p}{k_t+1}\,.
            \end{cases}
        \end{align*}
        We denote this distribution by $P_{\notin}(a_t)$.
    \end{enumerate}
To measure the information obtained from $\Delta_t$, we will use the KL divergence of the distribution of $\Delta_t$
 under the data distribution $\cD_{i^*}'$ from that under a benchmark data distribution.
We use the average distribution over $\{\cD_i'|i\in [n]\}$, which is denoted by $\bar \cD = \frac{1}{n}\sum_{i\in n}\cD_i'$.
The sampling process is equivalent to drawing $i^*\sim \Unif([n])$ first and then sampling from $\cD'_{i^*}$.
Under $\bar \cD$, for any $j\in a_t$, we have
    \begin{align*}
        \Pr(\Delta_t = j) 
        =& \Pr(i^*\in a_t\setminus\{j\})\Pr(\Delta_t = j|i^*\in a_t\setminus\{j\})+ \Pr(i^*=j)\Pr(\Delta_t = j|i^*=j) \\
        &+ \Pr(i^*\notin a_t)\Pr(\Delta_t = \be_j|i^*\notin a_t) \\&=\frac{k_t-1}{n}\cdot \frac{1-p}{k_t}+\frac{1}{n}\cdot p + \frac{n-k_t}{n}\cdot \frac{1-p}{k_t+1} = \frac{(nk_t-1)(1-p)}{nk_t(k_t+1)} +\frac{p}{n}\,,
    \end{align*}
    and
    \begin{align*}
        \Pr(\Delta_t = 0) &= \Pr(i^*\in a_t)\Pr(\Delta_t = 0|i^*\in a_t)+ \Pr(i^*\notin a_t)\Pr(\Delta_t = 0|i^*\notin a_t)\\&=\frac{k_t}{n}\cdot \frac{1-p}{k_t} + \frac{n-k_t}{n}\cdot (p+ \frac{1-p}{k_t+1}) = \frac{(n+1)(1-p)}{n(k_t+1)} + \frac{(n-k_t)p}{n}\,.
    \end{align*}
    Thus, the distribution of $\Delta_t$ under $\bar\cD$ is
    \begin{align*}
            \Delta_t = \begin{cases}
                j & \text{w.p. } \frac{(nk_t-1)(1-p)}{nk_t(k_t+1)} +\frac{p}{n}\text{ for } j \in a_t\\
                0 & \text{w.p. } \frac{(n+1)(1-p)}{n(k_t+1)} + \frac{(n-k_t)p}{n}\,.
            \end{cases}
    \end{align*}
    We denote this distribution by $\bar P(a_t)$.
    Next we will compute the KL divergence of $P_{\notin}(a_t)$ and $P_{\in}(a_t)$ from $\bar P(a_t)$.
    Since $p =\frac{\epsilon}{16n^2}\leq \frac{1}{16n^2}$, we have $\frac{(nk_t-1)(1-p)}{nk_t(k_t+1)} +\frac{p}{n} \leq \frac{1-p}{k_t+1}$ and $\frac{(n+1)(1-p)}{n(k_t+1)} + \frac{(n-k_t)p}{n} \leq \frac{1}{k_t}+p$.
    We will also use $\log(1+x)\leq x$ for $x\geq 0$ in the following calculation.
    For any $i^*\in a_t$, we have
    \begin{align}
        &\KL{\bar P(a_t)}{P_{\in}(a_t,i^*)} \nonumber\\
        =&(k_t-1)\left(\frac{(nk_t-1)(1-p)}{nk_t(k_t+1)} +\frac{p}{n}\right) \log\left((\frac{(nk_t-1)(1-p)}{nk_t(k_t+1)} +\frac{p}{n}) \cdot \frac{k_t}{1-p}\right) \nonumber\\
        &+ \left(\frac{(nk_t-1)(1-p)}{nk_t(k_t+1)} +\frac{p}{n}\right) \log\left((\frac{(nk_t-1)(1-p)}{nk_t(k_t+1)} +\frac{p}{n}) \cdot \frac{1}{p}\right) \nonumber\\
        &+\left( \frac{(n+1)(1-p)}{n(k_t+1)} + \frac{(n-k_t)p}{n}\right) \log \left(\left( \frac{(n+1)(1-p)}{n(k_t+1)} + \frac{(n-k_t)p}{n}\right) \cdot \frac{k_t}{1-p}\right)\nonumber\\
        \leq & (k_t-1)\left(\frac{(nk_t-1)(1-p)}{nk_t(k_t+1)} +\frac{p}{n}\right)\log(\frac{1-p}{k_t+1}\cdot \frac{k_t}{1-p}) + \frac{1-p}{k_t+1} \log(1\cdot\frac{1}{p}) \nonumber\\
        &+ (\frac{1}{k_t}+p)\cdot \log\left(1+pk_t\right)\nonumber\\
        \leq & 0 + \frac{1}{k_t+1}\log(\frac{1}{p}) + \frac{2}{k_t}\cdot p k_t =\frac{1}{k_t+1}\log(\frac{1}{p}) + 2p 
        \,.\label{eq:kl-in}
        \end{align}
        For $P_{\notin}(a_t)$, we have
        \begin{align}
            &\KL{\bar P(a_t)}{P_{\notin}(a_t)}\nonumber\\
            =&k_t\left(\frac{(nk_t-1)(1-p)}{nk_t(k_t+1)} +\frac{p}{n}\right) \log\left((\frac{(nk_t-1)(1-p)}{nk_t(k_t+1)} +\frac{p}{n}) \cdot \frac{k_t+1}{1-p}\right) \nonumber\\
        &+\left( \frac{(n+1)(1-p)}{n(k_t+1)} + \frac{(n-k_t)p}{n}\right) \log \left(\left( \frac{(n+1)(1-p)}{n(k_t+1)} + \frac{(n-k_t)p}{n}\right) \cdot \frac{1}{p+ \frac{1-p}{k_t+1}}\right)\nonumber\\
        \leq & k_t \left(\frac{(nk_t-1)(1-p)}{nk_t(k_t+1)} +\frac{p}{n}\right)\log(\frac{1-p}{k_t+1}\cdot \frac{k_t+1}{1-p}) \nonumber\\
        &+ (\frac{1}{k_t}+p)\log \left(\left( \frac{(n+1)(1-p)}{n(k_t+1)} + \frac{(n-k_t)p}{n}\right) \cdot \frac{1}{p+ \frac{1-p}{k_t+1}}\right)\nonumber\\
        = & 0+ (\frac{1}{k_t}+p) \log(1+ \frac{1-p(k_t^2+k_t+1)}{n(1+k_tp)})\nonumber
        \\
        \leq &(\frac{1}{k_t}+p) \frac{1}{n(1 + k_tp)}
        = \frac{1}{nk_t}\label{eq:kl-out}\,.
        \end{align}

\paragraph{Lower bound of the information} Now we adopt the similar framework used in the proofs of Theorem~\ref{thm:delta-csv} and \ref{thm:x-delta-never}.
For notation simplicity, for all $i\in [n]$, let $\bP_i$ denote the dynamics of $(x_1,f_1,\Delta_1, y_1,\hat y_1,\ldots, x_T, f_T,\Delta_T, y_T,\hat y_T)$ under $\cD_i'$ and and $\bar \bP$ denote the dynamics under $\bar \cD$.
Let $B_t$ denote the event of $\{f_t = 2\ind{a_t}-1 \text{ for some non-empty } a_t\subset [n]\}$.
As discussed before, for any $a_t$, conditional on $\neg B_t$ or $y_t =+1$,  $(x_t, \Delta_t, y_t,\hat y_t)$ are identical in all $\{\cD_i'|i\in [n]\}$, and therefore, also identical in $\bar \cD$.
We can only obtain information at rounds when $B_t \wedge (y_t =-1)$ occurs.
In such rounds, we know that $x_t$ is always $0$, $f_t$ is fully determined by history (possibly with external randomness , which does not depend on data distribution), $y_t =-1$ and $\hat y_t$ is fully determined by $\Delta_t$ ($\hat y_t = +1$ iff. $\Delta_t \neq 0$).

Therefore, conditional the history $H_{t-1} = (x_1,f_1,\Delta_1, y_1,\hat y_1,\ldots, x_{t-1}, f_{t-1},\Delta_{t-1}, y_{t-1},\hat y_{t-1})$ before time $t$, we have
\begin{align}
    &\KL{\bar \bP(x_{t}, f_{t},\Delta_{t}, y_{t},\hat y_{t}|H_{t-1})}{\bP_i(x_{t}, f_{t},\Delta_{t}, y_{t},\hat y_{t}|H_{t-1})} \nonumber\\
    =&\bar \bP(B_t  \wedge y_t =-1) \KL{\bar \bP(\Delta_t|H_{t-1},B_t  \wedge y_t =-1)}{\bP_i(\Delta_t|H_{t-1},B_t  \wedge y_t =-1)} \nonumber\\
    =& 6\epsilon \bar \bP(B_t) \KL{\bar \bP(\Delta_t|H_{t-1},B_t  \wedge y_t =-1)}{\bP_i(\Delta_t|H_{t-1},B_t  \wedge y_t =-1)}\label{eq:kl-delta}\,,
\end{align}
where the last equality holds due to that $y_t\sim \Rad(1-6\epsilon)$ and does not depend on $B_t$.

For any algorithm that can successfully identify $i$ under the data distribution $\cD_i$ with probability $\frac{3}{4}$ for all $i\in [n]$, then $\bP_{\cD_i}(i_\out = i)\geq \frac{3}{4}$ and $\bP_{\cD_j}(i_\out = i)\leq \frac{1}{4}$ for all $j\neq i$.
Recall that $\cD_i$ and $\cD_i'$ are very close when the mixture parameter $p$ is small. Combining with Eq~\eqref{eq:smooth-res}, we have 
\begin{align*}
    &\abs{\bP_{i}(i_\out = i) - \bP_{j}(i_\out = i)}\\
    \geq& \abs{\bP_{\cD_i}(i_\out = i) - \bP_{\cD_j}(i_\out = i)} - \abs{\bP_{\cD_i}(i_\out = i)- \bP_{i}(i_\out = i)}- \abs{\bP_{\cD_j}(i_\out = i)- \bP_{j}(i_\out = i)} \\
    \geq&\frac{1}{2}- \frac{1}{4} = \frac{1}{4}\,.
\end{align*}

Then we have the total variation distance between $\bP_{i}$ and $\bP_{j}$
\begin{align}
       \TV(\bP_{i},\bP_{j})\geq \abs{\bP_{i}(i_\out = i) - \bP_{j}(i_\out = i)}\geq \frac{1}{4}\,.\label{eq:tv-lb}
\end{align}

Then we have
    \begin{align*}
    &\EEs{i\sim \Unif([n])}{\TV^2(\bP_{i},\bP_{(i+1)\text{ mod } n})} \leq 4 \EEs{i\sim \Unif([n])}{\TV^2(\bP_{i},\bar \bP )}\\
    \leq& 2\EEs{i}{\KL{\bar \bP}{\bP_{i}}} \tag{Pinsker's ineq}\\
    =& 2\EEs{i}{\sum_{t=1}^T \KL{\bar \bP(x_{t}, f_{t},\Delta_{t}, y_{t},\hat y_{t}|H_{t-1})}{\bP_i(x_{t}, f_{t},\Delta_{t}, y_{t},\hat y_{t}|H_{t-1})} }\tag{Chain rule}\\
    \leq & 12\epsilon\EEs{i}{\sum_{t=1}^T \bar \bP(B_t) \KL{\bar \bP(\Delta_t|H_{t-1},B_t  \wedge y_t =-1)}{\bP_i(\Delta_t|H_{t-1},B_t  \wedge y_t =-1)}}\tag{Apply Eq~\eqref{eq:kl-delta}}\\
     \leq& \frac{12\epsilon}{n} \sum_{t=1}^T \bar \bP(B_t)\sum_{i=1}^n \KL{\bar \bP(\Delta_t|H_{t-1},B_t  \wedge y_t =-1)}{\bP_i(\Delta_t|H_{t-1},B_t  \wedge y_t =-1)}\\
     =& \frac{12\epsilon}{n}  \EEs{a_{1:T}\sim \bar \bP}{\sum_{t=1}^T\true{B_t}\left(\sum_{i:i\in a_t} \KL{\bar P(a_t)}{P_{\in}(a_t)}
     + \sum_{i:i\notin a_t}\KL{\bar P(a_t)}{P_{\notin}(a_t)}\right)}\\
     \leq & \frac{12\epsilon}{n}  \EEs{a_{1:T}\sim \bar \bP}{\sum_{t:\true{B_t} =1}\left(\sum_{i:i\in a_t} \left(\frac{1}{k_t+1}\log(\frac{1}{p}) + 2p
     \right) + \sum_{i:i\notin a_t}\frac{1}{n k_t}\right)}\tag{Apply Eq~\eqref{eq:kl-in},\eqref{eq:kl-out}}\\
     \leq & \frac{12\epsilon}{n} \sum_{t=1}^T (\log(\frac{1}{p})+ 2np + 1)\\
     \leq & \frac{12T\epsilon(\log(16n^2/\epsilon) + 2)}{n}\,.
\end{align*}
Combining with Eq~\eqref{eq:tv-lb}, we have that there exists a universal constant $c$ such that $T\geq \frac{cn}{\epsilon (\log(n/\epsilon) +1 )}$. 
\end{proof}

\end{document}